\documentclass{article}

\usepackage{microtype}
\usepackage{graphicx}
\usepackage{subcaption}
\usepackage{booktabs} %

\usepackage{nicefrac}       %
\usepackage{amsthm, amsmath, amssymb}
\usepackage{bm, url}

\usepackage{hyperref}
\usepackage{cleveref}
\usepackage{enumitem}

\usepackage[accepted]{icml2020}

\theoremstyle{plain}
\newtheorem{theorem}{Theorem}[section]
\newtheorem{lemma}[theorem]{Lemma}

\newtheorem{corollary}[theorem]{Corollary}
\theoremstyle{definition}
\newtheorem{definition}[theorem]{Definition}
\newtheorem{example}[theorem]{Example}
\newtheorem{remark}[theorem]{Remark}
\newtheorem{assumption}{Assumption}[section]

\newcommand{\trans}{\mathsf T}

\newcommand{\reffig}[1]{Fig.\,\ref{#1}}
\newcommand{\reftab}[1]{Table~\ref{#1}}
\newcommand{\refeq}[1]{\eqref{#1}}

\newcommand{\E}{\mathbb E}
\newcommand{\R}{\mathbb R}

\newcommand{\Z}{\mathbb Z}
\newcommand{\N}{\mathbb N}

\DeclareMathOperator*{\tr}{tr}

\DeclareMathOperator*{\argmin}{arg\,min\,}
\newcommand{\norm}[1]{\left \Vert #1 \right \Vert}
\newcommand{\normz}[1]{\Vert #1 \Vert}
\newcommand{\hsnorm}[1]{\norm{#1}_\mathrm{HS}}
\newcommand{\hsnormz}[1]{\normz{#1}_\mathrm{HS}}
\newcommand{\llnorm}[1]{\norm{#1}_{\rho}}
\newcommand{\llnormz}[1]{\normz{#1}_{\rho}}

\newcommand{\divger}{\mathrm{div}\,} 
\newcommand{\divgers}[1]{\mathrm{div}_{#1}\,}

\newcommand{\inner}[2]{\left \langle #1, #2 \right \rangle}
\newcommand{\innerz}[2]{\langle #1, #2 \rangle}

\newcommand{\bone}{\textbf{1}}

\newcommand{\ba}{\textbf{a}}
\newcommand{\bb}{\textbf{b}}
\newcommand{\bc}{\textbf{c}}

\newcommand{\be}{\textbf{e}}

\newcommand{\bh}{\textbf{h}}

\newcommand{\bzeta}{\zeta}

\newcommand{\br}{\textbf{r}}
\newcommand{\bs}{{s}}

\newcommand{\bu}{\textbf{u}}

\newcommand{\bw}{\textbf{w}}
\newcommand{\bx}{\textbf{x}}

\newcommand{\by}{\textbf{y}}
\newcommand{\bz}{\textbf{z}}

\newcommand{\bA}{\textbf{A}}

\newcommand{\bG}{\textbf{G}}
\newcommand{\bH}{\textbf{H}}
\newcommand{\bI}{\textbf{I}}

\newcommand{\bK}{\textbf{K}}
\newcommand{\bL}{\textbf{L}}

\newcommand{\bS}{\textbf{S}}

\newcommand{\bX}{\textbf{X}}
\newcommand{\bY}{\textbf{Y}}
\newcommand{\bZ}{\textbf{Z}}

\newcommand{\cX}{\mathcal{X}}

\newcommand{\cD}{\mathcal{D}}
\newcommand{\cE}{\mathcal{E}}
\newcommand{\cF}{\mathcal{F}}

\newcommand{\cH}{\mathcal{H}}

\newcommand{\cK}{\mathcal{K}}
\newcommand{\cKcf}{\mathcal{K}_\mathrm{cf}}
\newcommand{\cL}{\mathcal{L}}
\newcommand{\ltwospace}{\cL^2(\cX, \rho; \R^d)}

\newcommand{\cN}{\mathcal{N}}

\newcommand{\cP}{\mathcal{P}}

\newcommand{\cW}{\mathcal{W}}
\newcommand{\cY}{\mathcal{Y}}
\newcommand{\cZ}{\mathcal{Z}}

\begin{document}

\twocolumn[
\icmltitle{Nonparametric Score Estimators}

\icmlsetsymbol{equal}{*}

\begin{icmlauthorlist}
\icmlauthor{Yuhao Zhou}{to}
\icmlauthor{Jiaxin Shi}{to}
\icmlauthor{Jun Zhu}{to}
\end{icmlauthorlist}

\icmlaffiliation{to}{Dept. of Comp. Sci. \& Tech., BNRist Center, Institute for AI, Tsinghua-Bosch ML Center, Tsinghua University}

\icmlcorrespondingauthor{J. Zhu}{dcszj@tsinghua.edu.cn}

\icmlkeywords{Machine Learning, Kernel Methods, ICML}

\vskip 0.3in
]

\printAffiliationsAndNotice{}  %

\begin{abstract}
Estimating the score, i.e., the gradient of log density function, from a set of samples generated by an unknown distribution is a fundamental task in inference and learning of probabilistic models that involve flexible yet intractable densities. 
Kernel estimators based on Stein's methods or score matching have shown promise,
however their theoretical properties and relationships have not been fully-understood.
We provide a unifying view of these estimators under the framework of regularized nonparametric regression. 
It allows us to analyse existing estimators and construct new ones with desirable properties by choosing different hypothesis spaces and regularizers.
A unified convergence analysis is provided for such estimators.
Finally, we propose score estimators based on iterative regularization that enjoy computational benefits from curl-free kernels and fast convergence.

\end{abstract}

\section{Introduction}
\label{sec:intro}

Intractability of density functions has long been a central challenge in probabilistic learning. This may arise from various situations such as training implicit models like GANs~\citep{goodfellow2014generative}, or marginalizing over a non-conjugate hierarchical model, e.g., evaluating the output density of stochastic neural networks~\citep{sun2018functional}.
In these situations, inference and learning often require evaluating such intractable densities or optimizing an objective that involves them.

Among various solutions, one important family of methods are based on \emph{score estimation}, which rely on a key step of estimating the \emph{score}, i.e., the derivative of the log density $\nabla_{\bx}\log p(\bx)$ from a set of samples drawn from some unknown probability density $p$.
These methods include parametric score matching~\citep{hyvarinen2005estimation,sasaki2014clustering,song2019sliced}, its denoising variants as autoencoders~\citep{vincent2011connection}, nonparametric score matching~\citep{sriperumbudur2017density,sutherland2017efficient}, and kernel score estimators based on Stein's methods~\citep{li2018gradient,shi2018spectral}.
They have been successfully applied to applications such as estimating gradients of mutual information for representation learning~\citep{wen2020mutual}, score-based generative modeling~\citep{song2019generative,saremi2019neural}, gradient-free adaptive MCMC~\citep{strathmann2015gradient}, learning implicit models~\citep{warde2016improving}, and solving intractability in approximate inference algorithms~\citep{sun2018functional}.

Recently, nonparametric score estimators are growing in popularity, mainly because they are flexible, have well-studied statistical properties, and perform well when samples are very limited. 
Despite a common goal, they have different motivations and expositions. 
For example, the work~\citet{sriperumbudur2017density} is motivated from the density estimation perspective and the richness of kernel exponential families~\citep{canu2006kernel,fukumizu2009exponential}, where the estimator is obtained by score matching. \citet{li2018gradient} and \citet{shi2018spectral} are mainly motivated by Stein's methods. %
The solution of \citet{li2018gradient} gives the score prediction at sample points by minimizing the kernelized Stein discrepancy~\citep{chwialkowski2016kernel,liu2016kernelized} and at an out-of-sample point by adding it to the training data, while the estimator of \citet{shi2018spectral} is obtained by a spectral analysis in function space.

As these estimators are studied in different contexts, their relationships and theoretical properties are not fully-understood. In this paper, we provide a unifying view of them under the regularized nonparametric regression framework. This framework allows us to construct new estimators with desirable properties, and to justify the consistency and improve the convergence rate of existing estimators. It also allows us to clarify the relationships between these estimators. We show that they differ only in hypothesis spaces and regularization schemes.

Our contributions are both theoretical and algorithmic:
\vspace{-3.5mm}
\begin{itemize}[itemsep=0pt]
	\item We provide a unifying perspective of nonparametric score estimators. We show that the major distinction of the KEF estimator~\citep{sriperumbudur2017density} from the other two estimators lies in the use of curl-free kernels, while \citet{li2018gradient} and \citet{shi2018spectral} differ mostly in regularization schemes, with the former additionally ignores a one-dimensional subspace in the hypothesis space. We provide a unified convergence analysis under the framework.
	\item  We justify the consistency of the Stein gradient estimator~\citep{li2018gradient}, although the originally proposed out-of-sample extension is heuristic and expensive.
	We provide a natural and principled out-of-sample extension derived from our framework. 
	For both approaches we provide explicit convergence rates.
    \item From the convergence analysis we also obtain the explicit rate for \citet{shi2018spectral}, which can be shown to improve the error bound of \citet{shi2018spectral}.
	\item Our results suggest favoring curl-free estimators in high dimensions. 
	To address the scalability challenge, we propose iterative score estimators by adopting the $\nu$-method~\citep{engl1996regularization} as the regularizer. 
	We show that the structure of curl-free kernels can further accelerate such algorithms.
	Inspired by a similar idea, we propose a conjugate gradient solver of KEF that is significantly faster than previous approximations. 
\end{itemize}
\vspace{-.3cm}
\paragraph{Notation} We always assume $\rho$ is a probability measure with probability density function $p(\bx)$ supported on $\cX \subset \R^d$, and $\ltwospace$ is the Hilbert space of all square integrable functions $f: \cX \to \R^d$ with inner product $\inner{f}{g}_{\ltwospace} = \E_{\bx \sim \rho}[\inner{f(\bx)}{g(\bx)}_{\R^d}]$. We denote by $\inner{\cdot}{\cdot}_\rho$ and $\llnorm{\cdot}$ the inner product and the norm in $\ltwospace$, respectively.
We denote $k$ as a scalar-valued kernel, and $\cK$ as a matrix-valued kernel $\cK: \cX \times \cX \to \R^{d\times d}$ satisfying the following conditions: (1) $\cK(\bx, \bx^\prime) = \cK(\bx^\prime, \bx)^\trans$ for any $\bx, \bx^\prime \in \cX$; (2) $\sum_{i, j=1}^m \bc_i^\trans\cK(\bx_i, \bx_j)\bc_j \geq 0$ for any $\left\{ \bx_i \right\} \subset \cX$ and $\left\{ \bc_i \right\}\subset \R^d$.
We denote a vector-valued \textit{reproducing kernel Hilbert space}~(RKHS) associated to $\cK$ by $\cH_\cK$, which is the closure of $\left\{ \sum_{i=1}^m \cK(\bx_i, \cdot)\bc_i : \bx_i \in \cX, \bc_i \in \R^d, m \in \N \right\}$ under the norm induced by the inner product 
$\inner{\cK(\bx_i, \cdot)\bc_i}{\cK(\bx_j, \cdot)\mathbf s_j} := \bc_i^\trans\cK(\bx_i, \bx_j)\mathbf s_j$.
We define $\cK_\bx := \cK(\bx, \cdot)$ and $[M] := \{ 1, \cdots, M \}$ for $M \in \Z_+$. For $\bA_1, \cdots, \bA_n \in \R^{s\times t}$, we use $(\bA_1, \cdots, \bA_n)$ to represent a block matrix $\bA \in \R^{ns \times t}$ with $A_{(i-1)s+j, k}$ being the $(j, k)$-th component of $\bA_i$, and we similarly define $[\bA_1, \cdots, \bA_n] := (\bA_1^\trans, \cdots, \bA_n^\trans)^\trans$.

\section{Background}
\label{sec:background}

In this section, we briefly introduce the nonparametric regression method of learning vector-valued functions~\citep{baldassarre2012multi}. %
We also review existing kernel-based approaches to score estimation.

\subsection{Vector-Valued Learning}
\label{sec:multi-output-learning}
Supervised vector-valued learning amounts to learning a vector-valued function $f_\bz: \cX \to \cY$ from a training set $\bz = \{ (\bx^m, \by^m) \}_{m\in[M]}$, where $\cX \subseteq \R^d$, $\cY \subseteq \R^q$. %
Here we assume the training data is sampled from an unknown distribution $\rho(\bx, \by)$, which can be decomposed into $\rho(\by | \bx)\rho_\cX(\bx)$.
A criterion for evaluating such an estimator is the mean squared error (MSE) $\cE(f) := \E_{\rho(\bx, \by)} \norm{f(\bx) - \by}_2^2$. It is well-known that the conditional expectation $f_\rho(\bx) := \E_{\rho(\by | \bx)} [\by]$ minimizes $\cE$. 
In practice, we minimize the empirical error $\cE_\bz(f) := \frac{1}{M}\sum_{m=1}^M \norm{f(\bx^m) - \by^m}_2^2$
in a certain hypothesis space $\cF$. However, the minimization problem is typically ill-posed for large $\cF$. Hence, it is convenient to consider the regularized problem:
\begin{equation}
    \label{eqn:empirical-Tikhonov-regularized-regression}
     f_{\bz, \lambda} := \argmin_{f\in\cF} \cE_\bz(f) + \lambda \norm{f}_\cF^2,
\end{equation}
where $\norm{\cdot}_\cF$ is the norm in $\cF$. In the vector-valued case, it is typical to consider a vector-valued RKHS $\cH_\cK$ associated with a matrix-valued kernel $\cK$ as the hypothesis space. Then the estimator is $f_{\bz, \lambda} = \sum_{m=1}^M \cK_{\bx^m} \bc^m$, where $\cK_{\bx^m}$ denotes the function $\cK(\bx^m, \cdot)$.  $\bc^m$ solves the linear system $(\frac{1}{M}\bK + \lambda I)\bc = \frac{1}{M}\by$ with $\bK_{ij} = \cK(\bx^i, \bx^j), \bc = (\bc^1, \cdots, \bc^M), \by = (\by^1, \cdots, \by^M)$.

For convenience, we define the \textit{sampling operator} $S_\bx : \cH_\cK \to \R^{Mq}$ as $S_\bx(f) := (f(\bx^1), \cdots, f(\bx^M))$. 
Its adjoint $S_\bx^*: \R^{Mq} \to \cH_\cK$ that satisfies $\langle S_\bx(f), \bc\rangle_{\R^{Mq}} = \langle f, S_\bx^*(\bc)\rangle_{\cH_\cK}, \forall f\in \cH_\cK, \bc\in \mathbb{R}^{Mq}$ is $S_\bx^*(\bc^1, \cdots, \bc^M) = \sum_{m=1}^M \cK_{\bx^m}\bc^m$. 
Since $(\frac{1}{M}S^*_\bx S_\bx + \lambda I)f_{\bz,\lambda} = \frac{1}{M}S_\bx^*\bK\bc + \lambda S_\bx^*\bc = \frac{1}{M}S^*_\bx \by$, the estimator now can be written as $f_{\bz, \lambda} = \left ( \frac{1}{M}S_\bx^*S_\bx + \lambda I \right )^{-1} \frac{1}{M}S_\bx^*\by$.
In fact, if we consider the data-free limit of \eqref{eqn:empirical-Tikhonov-regularized-regression}: $\argmin_{f\in\cH_\cK}\cE(f) + \lambda \norm{f}_{\cH_\cK}^2$,
the minimizer is unique when $\lambda > 0$ and is given by $f_\lambda := (L_\cK + \lambda I)^{-1} L_\cK f_\rho$, where $L_\cK: \cH_\cK \to \cH_\cK$ is the \textit{integral operator} defined as $L_\cK f := \int_\cX \cK_\bx f(\bx)d\rho_\cX$~\citep{smale2007learning}. 
It turns out that $\frac{1}{M}S^*_\bx S_\bx$ is an empirical estimate of $L_\cK$:
$\hat L_\cK f := \frac{1}{M} \sum_{m=1}^M \cK_{\bx^m}f(\bx^m) = \frac{1}{M} S^*_\bx S_\bx f$. 
It can also be shown that $\hat L_\cK f_\rho = \frac{1}{M}S^*_\bx \by$. 
Hence, we can write $f_{\bz, \lambda} = (\hat L_\cK + \lambda I)^{-1} \hat L_\cK f_\rho$. 

As we have mentioned, the role of regularization is to deal with the ill-posedness. 
Specifically, $\hat L_\cK$ is not always invertible as it has finite rank and $\cH_\cK$ is usually of infinite dimension. 
Many regularization methods are studied in the context of solving inverse problems~\citep{engl1996regularization} and statistical learning theory~\citep{bauer2007regularization}. 
The regularization method we presented in \refeq{eqn:empirical-Tikhonov-regularized-regression} is the famous \emph{Tikhonov regularization}, which belongs to a class of regularization techniques called spectral regularization~\citep{bauer2007regularization}.
Specifically, spectral regularization
corresponds to a family of estimators defined as
\[ f_{\bz, \lambda}^g := g_\lambda(\hat L_\cK) \hat L_\cK f_\rho, \]
where $g_\lambda: \R^+ \to \R$ is a regularizer such that $g_\lambda(\hat L_\cK)$ approximates the inverse of $\hat L_\cK$. 
Note that $\hat L_\cK$ can be decomposed into $\sum \sigma_i \inner{e_i}{\cdot } e_i$, where $(\sigma_i, e_i)$ is a pair of eigenvalue and eigenfunction, we can define $g_\lambda(\hat L_\cK) := \sum g_\lambda(\sigma_i) \inner{e_i}{\cdot} e_i$.
The Tikhonov regularization corresponds to $g_\lambda(\sigma) = (\lambda + \sigma)^{-1}$.
There are several different regularizers. 
For example, the \textit{spectral cut-off regularizer} is defined by $g_\lambda(\sigma) = \sigma^{-1}$ for $\sigma \geq \lambda$ and $g_\lambda(\sigma) = 0$ otherwise. %
We refer the readers to \citet{smale2007learning,bauer2007regularization,baldassarre2012multi} for more details.

\subsection{Related Work}
\label{sec:score-estimators}

We assume $\log p(\bx)$ is differentiable, and define the \textit{score} as $\bs_p := \nabla \log p$. 
By score estimation we aim to estimate $\bs_p$ from a set of i.i.d. samples $\{\bx^m\}_{m\in[M]}$ drawn from $\rho$.
There have been many kernel-based score estimators studied in different contexts~\citep{sriperumbudur2017density,sutherland2017efficient,li2018gradient,shi2018spectral}.
Below we give a brief review of them.

\paragraph{Kernel Exponential Family Estimator}
\label{sec:kexpf}
The kernel exponential family (KEF)~\citep{canu2006kernel,fukumizu2009exponential} was originally proposed as an infinite-dimensional generalization of exponential families.
It was shown to be useful in density estimation as it can approximate a broad class of densities arbitrarily well~\citep{sriperumbudur2017density}.
The KEF is defined as:
\[ \cP_k := \{ p_f(\bx) =  e^{f(\bx) - A(f)} : f \in \cH_k, e^{A(f)} < \infty \}, \]
where $\cH_k$ is a scalar-valued RKHS, and $A(f) := \log \int_\cX e^{f(\bx)}dx$ is the normalizing constant.
Since $A(f)$ is typically intractable,
\citet{sriperumbudur2017density} proposed to estimate $f$ by matching the model score $\nabla \log p_f$ and the data score $\bs_p$, thus the KEF can naturally be used for score estimation~\citep{strathmann2015gradient}.
This approach works by
minimizing the regularized score matching loss: \begin{equation} \label{eq:kef-sm}
\min_{f\in \mathcal{H}_k}J(p\|p_f) + \lambda \norm{f}^2_{\cH_k},
\end{equation}
where $J(p \| q) := \E_p \norm{\nabla \log p - \nabla \log q}_2^2$ is the Fisher divergence between $p$ and $q$.
Integration by parts was used to eliminated $\nabla \log p$ from $J(p \| q)$~\citep{hyvarinen2005estimation} and
the exact solution of \eqref{eq:kef-sm} was given as follows~\citep[Theorem 5]{sriperumbudur2017density}:
\begin{equation}
    \label{eqn:kef-formula}
    \hat{f}_{p,\lambda} = \sum_{m=1}^M\sum_{j=1}^d c_{(m-1)d + j} \partial_j k(\bx^m, \cdot) -\frac{\hat \xi}{\lambda},
\end{equation}
where $\hat \xi(\bx) = \frac{1}{M}\sum_{m=1}^M\sum_{j=1}^d \partial_j^2 k(\bx^m, \cdot)$, $\bc\in \mathbb{R}^{Md}$ %
is obtained by solving $(\bG + M\lambda \mathbf{I})\bc = \bb / \lambda$ with $\bG_{(m-1)d+i,(\ell-1)d+j} = \partial_i\partial_{j+d} k(\bx^m,\bx^\ell)$ and $\bb_{(m-1)d+i} = \frac{1}{M}\sum_{\ell=1}^M \sum_{j=1}^d \partial_i\partial_{j+d}^2 k(\bx^m,\bx^\ell)$, where $\partial_{j+d}$ denotes taking derivative w.r.t. the $j$-th component of the second parameter $\bx^\ell$.
This solution suffers from computational drawbacks due to the large linear system of size $Md\times Md$.
\citet{sutherland2017efficient} proposed to use the Nystr\"om method to accelerate KEF. 
Instead of minimizing the loss in the whole RKHS, they minimized it in a low dimensional subspace. %

\paragraph{Stein Gradient Estimator}
The Stein gradient estimator proposed by~\citet{li2018gradient} is based on inverting the following generalized Stein's identity~\citep{stein1981estimation,gorham2015measuring}
\begin{equation}
\label{eqn:stein-identity}
\E_p[h(\bx)\nabla \log p(\bx)^\trans + \nabla h(\bx)] = 0,
\end{equation}
where $h: \cX \to \R^d$ is a test function satisfying some regularity conditions. An empirical approximation of the identity is $-\frac{1}{M}\bH\bS \approx \overline{\nabla_\bx h}$, where $\bH = (h(\bx^1), \cdots, h(\bx^M))\in \R^{d\times M}$, $\bS = (\nabla \log p(\bx^1), \cdots, \nabla \log p(\bx^M))\in \R^{M\times d}$ and $\overline{\nabla_\bx h} = \frac{1}{M}\sum_{m=1}^M \nabla_{\bx^m}h(\bx^m)$. \citet{li2018gradient} proposed to minimize
$ \norm{\overline{\nabla_\bx h}+\frac{1}{M}\bH\bS}_F^2 + \frac{\eta}{M^2}\norm{\bS}_F^2$ to estimate $\bS$,
where $\norm{\cdot}_F$ denotes the Frobenius norm. The kernel trick $k(\bx^i, \bx^j):= h(\bx^i)^\trans h(\bx^j)$ was then exploited to obtain the estimator.
From the above we only have score estimates at the sample points. \citet{li2018gradient} proposed a heuristic out-of-sample extension at $\bx$ by adding it to $\{\bx^m\}$ and recompute the minimizer. 
Such an approach is unjustified.
It is still unclear whether the estimator is consistent.

\paragraph{Spectral Stein Gradient Estimator}
The Spectral Stein Gradient Estimator (SSGE)~\citep{shi2018spectral} was derived by a spectral analysis of the score function.
Unlike \citet{li2018gradient} it was shown to have convergence guarantees and principled out-of-sample extension. 
The idea is to expand each component of the score in a scalar-valued function space $\cL^2(\cX, \rho)$: 
$ g_i(\bx) = \sum_{j=1}^\infty \beta_{ij}\psi_j(\bx), $
where $g_i$ is the $i$-th component of the score and $\{ \psi_j \}$ are the eigenfunctions of the integral operator $L_k f := \int_\cX k(\bx, \cdot)f(\bx)d\rho(\bx)$ associated with a scalar-valued kernel $k$.
By using the Stein's identity in ~\eqref{eqn:stein-identity} \citet{shi2018spectral} showed that $\beta_{ij} = -\E_\rho[\partial_i \psi_j(\bx)]$. The Nystr\"om method~\citep{baker1977numerical,williams2001using} was then used to estimate $\psi_j$: 
\begin{equation}
    \label{eqn:ssge-eigen}
     \hat \psi_j(\bx) = \frac{\sqrt{M}}{\lambda_j} \sum_{m=1}^M k(\bx, \bx^m)w_{jm},
\end{equation}
where $\{ \bx^m \}_{m \in [M]}$ are i.i.d. samples  drawn from $\rho$, $w_{jm}$ is the $m$-th component of the eigenvector that corresponds to its $j$-th largest eigenvalue of the kernel matrix constructed from $\{ \bx^m \}_{m \in [M]}$. The final estimator was obtained by truncating $g_i$ to $\sum_{j=1}^J \beta_{ij}\psi_j(\bx)$ and plugging in $\hat \psi_j$. \citet[Theorem 2]{shi2018spectral} provided an error bound of SSGE depending on $J$ and $M$.
However, %
the convergence rate is still unknown.

\section{Nonparametric Score Estimators}
\label{sec:estimators}

The kernel score estimators discussed in Sec.~\ref{sec:score-estimators} were proposed in different contexts. 
The KEF estimator is motivated from the density estimation perspective, while Stein and SSGE have no explicit density models.
SSGE relies on spectral analysis in the function space, while the other two are derived by minimizing a loss function.
Despite sharing a common goal, it is still unclear how these estimators relate to each other.
In this section, we present a unifying framework of score estimation using regularized vector-valued regression.
We show that several existing kernel score estimators are special cases under the framework, which allows us to thoroughly investigate their strengths and weaknesses.

\subsection{A Unifying Framework}
\label{sec:unifying-framework}
As introduced in Sec.~\ref{sec:kexpf}, the goal is to estimate the score $\bs_p$ from a set of i.i.d. samples $\{\bx^m\}_{m\in[M]}$ drawn from $\rho$. 
We first consider the ideal case where we have the ground truth values of $\bs_p$ at the sample locations.
Then we can estimate $\bs_p$ with vector-valued regression as described in Sec.~\ref{sec:multi-output-learning}:
\begin{equation}
\label{eqn:tikhonov-minimization}
\hat{\bs}_{p,\lambda} = \argmin_{\bs \in \cH_\cK} \frac{1}{M}\sum_{m=1}^M \|\bs(\bx^m) - \bs_p(\bx^m)\|_2^2  + \frac{\lambda}{2}\|\bs\|^2_{\mathcal{H}_{\cK}}.
\end{equation}
The solution is given by $\hat{\bs}_{p,\lambda} = (\hat{L}_\cK + \lambda I)^{-1}\hat{L}_\cK\bs_p$. 
We could replace the Tikhonov regularizer with other spectral regularization, for which the general solution is
\begin{equation}
    \label{eqn:regularized-score-estimator}
    \hat \bs_{p, \lambda}^g := g_\lambda(\hat  L_\cK) \hat  L_\cK\bs_p.
\end{equation}
In reality, the values of $\bs_p$ at $\bx^{1:M}$ are unknown and we cannot compute $\hat{L}_\cK\bs_p$ as $\frac{1}{M}\sum_{m=1}^M \cK_{\bx^m}\bs_p(\bx^m)$. 
Fortunately, we could exploit integration by parts to avoid this problem.
Under some mild regularity conditions (Assumptions~\ref{assumption:domain}-\ref{assumption:integration-by-parts}), we have
    \[  L_\cK \bs_p = \E_\rho [\cK_\bx \nabla \log p(\bx)] = -\E_\rho [\divgers{\bx}\cK_\bx^\trans], \]
where the divergence of $\cK_\bx^\trans$ is defined as a vector-valued function, whose $i$-th component is the divergence of the $i$-th column of $\cK_\bx^\trans$. The empirical estimate $\hat  L_\cK \bs_p$ is then available as $-\frac{1}{M}\sum_{m=1}^M \divgers{\bx^m}\cK_{\bx^m}^\trans$,
which leads to the following general formula of nonparametric score estimators:
\begin{equation}
    \label{eqn:general-score-estimator}
    \hat \bs_{p,\lambda}^g = -g_\lambda(\hat L_\cK)\hat \bzeta,
\end{equation}
where $\hat \bzeta := \frac{1}{M}\sum_{m=1}^M \divgers{\bx^m}\cK_{\bx^m}^\trans$.

\subsection{Regularization Schemes}
\label{sec:regularization-schemes}

We now derive the final form of the estimator under three regularization schemes~\citep{bauer2007regularization}.
The choice of regularization will impact the convergence rate of the estimator, which will be studied in Sec.~\ref{sec:theory}.

\begin{theorem}[Tikhonov Regularization]
	\label{thm:tik-score}
	Let $\hat{\bs}^g_{p,\lambda}$ be defined as in \eqref{eqn:general-score-estimator}, and $g_\lambda(\sigma) = (\sigma + \lambda)^{-1}$. Then
	\begin{equation}
	\label{eqn:tikhonov-regularized-estimator}
	\hat \bs_{p,\lambda}^g(\bx) = \bK_{\bx\bX}\bc - \hat\bzeta(\bx) / \lambda,
	\end{equation}
	where $\bc$ is obtained by solving 
	\begin{equation}
	\label{eq:tik-inverse}
		(\bK + M\lambda I)\bc = \bh / \lambda.
	\end{equation}
	Here $\bc \in \mathbb{R}^{Md}$, $\bh = (\hat\bzeta(\bx^1), \cdots, \hat\bzeta(\bx^M)) \in \mathbb{R}^{Md}$, $\bK_{\bx\bX}  = [\cK(\bx, \bx^1), \cdots, \cK(\bx, \bx^M)] \in \mathbb{R}^{d\times Md}$, and $\bK \in \mathbb{R}^{Md\times Md}$ is given by $\bK_{(m - 1)d + i, (\ell - 1)d + j} = \cK(\bx^m, \bx^\ell)_{ij}$.
\end{theorem}

The proof is given in \cref{sec:proofs}, where the general representer theorem~\citep[Theorem A.2]{sriperumbudur2017density} is used to show that the solution lies in the subspace generated by
\begin{equation}
\label{eqn:tikhonov-solution-space}
\{ \cK_{\bx^m}\bc_m : m \in [M], \bc_m \in \R^d \} \cup \{ \hat \bzeta \}.
\end{equation}

Unlike the Tikhonov regularizer that shifts all eigenvalues simultaneously, the \emph{spectral cut-off regularization} sets
$g_\lambda(\sigma) = \sigma^{-1}$ for $\sigma \geq \lambda$, and $g_\lambda(\sigma) = 0$ otherwise. 
To obtain such estimator,
we need the following lemma that relates the spectral properties of $\bK$ and $\hat L_\cK$.

\begin{lemma}
\label{lemma:eigen-connection}
    Let $\sigma$ be a non-zero eigenvalue of $\frac{1}{M}\bK$ such that $\frac{1}{M}\bK\bu = \sigma \bu$, where $\bu \in \R^{Md}$ is the unit eigenvector. Then $\sigma$ is an eigenvalue of $\hat L_\cK$ and the corresponding unit eigenfunction is
    \[ v = \frac{1}{\sqrt{M\sigma}} \sum_{m=1}^M \cK_{\bx^m} \bu^{(m)}, \]
    where $\bu$ is splitted into $(\bu^{(1)}, \cdots, \bu^{(M)})$ and $\bu^{(i)} \in \R^d$.
\end{lemma}

The lemma is a direct generalization of \citet[Proposition 9]{rosasco2010learning} to vector-valued operators.

\begin{theorem}[Spectral Cut-Off Regularization]
    \label{thm:spectral-score}
    Let $\hat{\bs}_{p,\lambda}^g$ be defined as in \eqref{eqn:general-score-estimator}, and
    $$
    g_\lambda(\sigma) = \begin{cases}
    \sigma^{-1} & \sigma > \lambda, \\
    0 & \sigma \leq \lambda.
    \end{cases}
    $$
    Let $(\sigma_j, \bu_j)_{j\geq 1}$ be the eigenvalue and eigenvector pairs that satisfy $\frac{1}{M}\bK\bu_j = \sigma_j\bu_j$.
    Then we have
    \begin{equation}
    \label{eqn:spectral-cutoff-regularized-estimator}
    \hat \bs_{p,\lambda}^g(\bx) = -\bK_{\bx \bX}\left( \sum_{\sigma_j \geq \lambda} \frac{\bu_j\bu_j^\trans}{M\sigma_j^2}\right)\bh,
    \end{equation}
    where $\bK_{\bx\bX}$ and $\bh$ are defined as in \Cref{thm:tik-score}.
\end{theorem}

Apart from the above methods with closed-form solutions, 
early stopping of iterative solvers like gradient descent can also play the role of regularization~\citep{engl1996regularization}. 
Iterative methods replace the expensive inversion or eigendecomposition of the $Md\times Md$ size kernel matrix with fast matrix-vector multiplication.
In Sec.~\ref{sec:scalability} we show that such methods can be further accelerated by utilizing the structure of our kernel matrix.

We consider two iterative methods: the Landweber iteration and the $\nu$-method~\citep{engl1996regularization}. 
The Landweber iteration solves $\hat L_\cK s_p = -\hat\zeta$ with the fixed-point iteration:
\begin{equation}
    \label{eqn:landweber}
    \hat s_p^{(t + 1)} := \hat s_p^{(t)} - \eta\left(\hat \zeta + \hat L_\cK \hat s_p^{(t)}\right ), 
\end{equation}
where $\eta$ is a step-size parameter.
It can be regarded as using the following regularization:
\begin{theorem}[Landweber Iteration]
    \label{thm:landweber-score}
    Let $\hat{\bs}_{p,\lambda}^g$ and $\hat s^{(k)}_p$ be defined as in \eqref{eqn:general-score-estimator} and \eqref{eqn:landweber}, respectively. 
    Let $\hat s^{(0)} = 0$ and 
    $g_\lambda(\sigma) = \eta \sum_{i=0}^{t - 1} ( 1 - \eta \sigma)^i$, where %
    $t := \lfloor \lambda^{-1} \rfloor$.
    Then, 
    \[ \hat s_{p,\lambda}^g = \hat \bs_p^{(t)} = -t\eta \hat\zeta + \bK_{\bx\bX}\bc_t, \]
    where $\bc_0 = 0$, $\bc_{t+1} = (\mathbf{I}_d - \eta \bK / M)\bc_t - t\eta^2 \bh / M$, and $\bK, \bK_{\bx\bX}$, $\bh$ are defined as in \Cref{thm:tik-score}.
\end{theorem}

The Landweber iteration often requires a large number of iterations. An accelerated version of it is the $\nu$-method, where $\nu$ is a parameter controlling the maximal convergence rate (see Sec.~\ref{sec:theory}). 
The regularizer of the $\nu$-method can be represented by a family of polynomials $g_\lambda(\sigma) = \mathrm{poly}(\sigma)$. %
These polynomials approximate $1/\sigma$ better than those in the Landweber iteration.
As a result, the $\nu$-method only requires a polynomial of degree $\lfloor \lambda^{-1/2} \rfloor$ to define $g_\lambda$, 
which significantly reduces the number of iterations~\citep{engl1996regularization,bauer2007regularization}.
The next iterate of the $\nu$-method can be generated by the current and the previous ones:
\[    \hat \bs_{p}^{(t + 1)} = \hat \bs_{p}^{(t)}
        + u_t(\hat \bs_p^{(t)} - \hat \bs_p^{(t-1)})
        - \omega_t(\hat\bzeta + \hat L_\cK \hat \bs_p^{(t)}), \]
where $u_t, \omega_t$ are carefully chosen constants~\citep[Algorithm 6.13]{engl1996regularization}.
We describe the full algorithm in \Cref{example:nu-method}~(\cref{appendix:iterative-regularization}).

\subsection{Hypothesis Spaces}
\label{sec:hypothesis-spaces}
In this framework, the hypothesis space is characterized by the matrix-valued kernel that induces the RKHS~\citep{alvarez2012kernels}. 
Below we discuss two choices of the kernel: the diagonal ones are computationally more efficient, while curl-free kernels capture the conservative property of score vector fields.
\vspace{-.2cm}
\paragraph{Diagonal Kernels}
\label{sec:diagonal-kernels}
The simplest way to define a diagonal matrix-valued kernel is $\cK(\bx, \by) = k(\bx, \by) \bI_d$, where $k:\cX\times\cX\to\R$ is a scalar-valued kernel. 
This induces a product RKHS $\cH_k^d := \otimes_{i=1}^d \cH_k$ where all output dimensions of a function are independent. 
In this case the kernel matrix for $\bX = (\bx^1, \dots, \bx^M)$ is $\bK = k(\bX,\bX)\otimes \bI_d$, where $k(\bX,\bX)$ denotes the Gram matrix of the scalar-valued kernel $k$.
Therefore, the computational cost of matrix inversion and eigendecomposition is the same as in the scalar-valued case.
On the other hand, the independence assumption may not hold for score functions, whose output dimensions are correlated as they form the gradient of the log density. 
As we shall see in Sec.~\ref{sec:experiments}, such misspecification of the hypothesis space degrades the performance in high dimensions.

\vspace{-.2cm}
\paragraph{Curl-Free Kernels}
\label{sec:curl-kernels}
Noticing that score vector fields are gradient fields, we can use curl-free kernels~\citep{fuselier2007refined,macedo2010learning} to capture this property.
A curl-free kernel can be constructed from the negative Hessian of a translation-invariant kernel $k(\bx,\by) = \phi(\bx - \by)$:
$\cKcf(\bx, \by) := -\nabla^2\phi(\bx - \by)$,
where $\phi: \cX \to \R \in C^2$.
It is easy to see that $\cKcf$ is positive definite. 
A nice property of $\cH_{\cKcf}$ is that any element in it is a gradient of some function. 
To see this, notice that the $j$-th column of $\cKcf$ is $-\nabla(\partial_j\phi)$ and each element in $\cH_{\cKcf}$ is a linear combination of columns of $\cKcf$.
We also note that the unnormalized log-density function can be recovered from the estimated score when using curl-free kernels (see appendix~\ref{appendix:curl-free}).
The cost of inversion and eigendecomposition of the kernel matrix is $O(M^3d^3)$, compared to $O(M^3)$ for diagonal kernels.

\subsection{Examples}
\label{sec:connection}
In the following we provide examples of nonparametric score estimators derived from the framework.
We show that existing estimators can be recovered with certain types of kernels and regularization schemes (\Cref{tab:algorithms}).

\begin{example}[KEF]
    \label{example:kef}
Consider using curl-free kernels for the Tikhonov regularized estimator in \eqref{eqn:tikhonov-regularized-estimator}.
By substituting $\cKcf(\bx, \by) = -\nabla^2\phi(\bx - \by)$ for $\cK$, we get
\begin{equation*}
	\hat{\bs}_{p,\lambda}^g(\bx) =  -\sum_{m=1}^M\sum_{j=1}^d c_{(m-1)d + j} \nabla \partial_j \phi(\bx - \bx^m) - \frac{\hat \bzeta_{\mathrm{cf}}(\bx)}{\lambda},
\end{equation*}
where $\bzeta_{\mathrm{cf}}(\bx)_i := -\frac{1}{M}\sum_{m=1}^M\sum_{j=1}^d\partial_i\partial_j^2\phi(\bx - \bx^m)$.
Noticing that $\cKcf(\bx,\by)_{ij} = -\partial_i\partial_j \phi(\bx - \by) = \partial_i\partial_{j+d} k(\bx, \by)$, 
we could check that the definition of $\bc$ here, which follows from~\eqref{eqn:tikhonov-regularized-estimator}, is the same as in \eqref{eqn:kef-formula}. Thus by comparing with \eqref{eqn:kef-formula}, we have
\begin{equation}
	\hat{\bs}_{p,\lambda}^g(\bx) = \nabla\hat{f}_{p,\lambda}(\bx) = \nabla\log p_{\hat{f}_{p,\lambda}}(\bx).
\end{equation}
Therefore, \textit{the KEF estimator is equivalent to choosing curl-free kernels and the Tikhonov regularization in \eqref{eqn:general-score-estimator}.}
\end{example}

We note that, although the solutions are equivalent, the space $\{\nabla \log p_f, f\in \mathcal{H}_k\}$ looks different from the curl-free RKHS constructed from the negative Hessian of $k$.
Such equivalence of regularized minimization problems may be of independent interest.

\begin{example}[SSGE]
    \label{example:ssge}
    For the estimator \eqref{eqn:spectral-cutoff-regularized-estimator} obtained from the spectral cut-off regularization.
    Consider letting $\cK(\bx,\by) = k(\bx,\by)\bI_d$. 
    Then $\bK = k(\bX, \bX) \otimes \mathbf I_d$, and it can be decomposed as $\sum_{m=1}^M\sum_{i=1}^d \lambda_m (\bw_m\bw_m^\trans \otimes \be_i\be_i^\trans)$, where $\{ (\lambda_m, \bw_m) \}$ is the eigenpairs of $k(\bX, \bX)$ with $\lambda_1 \geq \lambda_2 \geq \dots \geq \lambda_M$ and $\{ \be_i \}$ is the standard basis of $\R^d$. 
    The estimator reduces to
    \begin{equation}
        \label{eqn:ssge-matrix-form}
    	 \hat \bs_{p,\lambda}^g(\bx)_i = -k(\bx,\bX) \left ( \sum_{\lambda_j \geq \lambda} \frac{\bw_j\bw_j^\trans}{\lambda_j^2} \right ) \br_i,
    \end{equation}
	where $\frac{1}{M}\br_i := (h_i, h_{d + i}, \cdots, h_{(M-1)d + i})$. 
    When we choose $\lambda = \lambda_J$,
    simple calculations (see \cref{appendix:computational-details}) show that \eqref{eqn:ssge-matrix-form} equals the SSGE estimator $\hat \bs_{p,\lambda}^g(\bx)_i = -\frac{1}{M} \sum_{j=1}^J \sum_{m=1}^M \partial_i \hat\psi_j(\bx^m) \hat\psi_j(\bx)$, where $\hat \psi_j$ is defined as in \refeq{eqn:ssge-eigen}. %
    Therefore, \textit{SSGE is equivalent to choosing the diagonal kernel $\cK(\bx,\by) = k(\bx,\by)\bI_d$ and the spectral cut-off regularization in \eqref{eqn:general-score-estimator}.}

\end{example}

\begin{example}[Stein]
    \label{example:stein}
    We consider modifying the Tikhonov regularizer to $g_\lambda(\sigma) = (\lambda + \sigma)^{-1} \bone_{\{ \sigma > 0 \}}$. In this case, we obtain an estimator $\hat \bs_{p,\lambda}^g(\bx) = -\bK_{\bx\bX}\bK^{-1} (\frac{1}{M}\bK + \lambda I)^{-1} \bh$ by \Cref{lemma:general-non-zero-regularizer}. At sample points, the estimated score is $-(\frac{1}{M}\bK + \lambda I)^{-1}\bh$, which coincides with the Stein gradient estimator. This suggests a principled out-of-sample extension of the Stein gradient estimator.

    To gain more insights, we consider to minimize \refeq{eqn:tikhonov-minimization} in the subspace generated by $\{ \cK_{\bx^m}\bc_m : m \in [M], \bc_m \in \R^d \}$. Compared with \refeq{eqn:tikhonov-solution-space}, the one-dimensional subspace $\R\hat\bzeta$ is ignored. We could check that (in \cref{appendix:computational-details}) this is equivalent to exploiting the previous mentioned regularizer. Therefore, \textit{the Stein estimator is equivalent to using the diagonal kernel $\cK(\bx,\by)=k(\bx,\by)\bI_d$ and the Tikhonov regularization with a one-dimensional subspace ignored.} 

    All the above examples can be extended to use a subset of the samples with Nystr{\"o}m methods~\citep{williams2001using}.
    Specifically, we can modify the general formula in \eqref{eqn:general-score-estimator} as $\hat s^{g,\bZ}_{p,\lambda} = -g_\lambda(P_\bZ \hat L_\cK P_\bZ) P_\bZ\hat \zeta$, where $P_\bZ: \cH_{\cK}\to \cH_{\cK}$ is the projection onto a low-dimensional subspace generated by the subset $\bZ$.
    When the curl-free kernel and the same truncated Tikhonov regularizer as in \Cref{example:stein} are used, this estimator is equivalent to the Nystr\"om KEF (NKEF)~\citep{sutherland2017efficient}.
    More details can be found in appendix~\ref{appendix:general-nkef}.
\end{example}

\subsection{Scalability}
\label{sec:scalability}
When using curl-free kernels, we need to deal with an $Md\times Md$ matrix.
In such cases, the Tikhonov and the spectral cut-off regularization cost $O(M^3d^3)$ and have difficulties scaling with the sample size and the input dimensions.
Fortunately, as the unifying perspective suggests, we could modify the regularization schemes with iterative methods that only require matrix-vector multiplications, e.g., the Landweber iteration and the $\nu$-method (see Sec.~\ref{sec:regularization-schemes}).
Interestingly, we get further acceleration by utilizing the structure of curl-free kernels.

\begin{example}[Iterative curl-free estimators]
	\label{example:iter}
We observe that when using a curl-free kernel $\cKcf$ constructed from a radial scalar-valued kernel $k(\bx, \by) = \phi(\norm{\bx - \by})$,
\[	\cKcf(\bx, \by) = \left( \frac{\phi^\prime}{r^3} - \frac{\phi^{\prime\prime}}{r^2} \right) \br\br^\trans - \frac{\phi^\prime}{r} \mathbf I, \]
where $\br = \bx - \by$, $r = \norm{\br}_2$. 
Consider in matrix-vector multiplications, for a vector $\ba \in \R^{d}$, $\cKcf(\bx,\by)\ba$ can be computed as 
$	\left( \frac{\phi^\prime}{r^3} - \frac{\phi^{\prime\prime}}{r^2} \right) (\br^\trans\ba) \br - \frac{\phi^\prime}{r} \ba,$
where only a vector-vector multiplication is required with time complexity $O(d)$, compared to general $O(d^2)$. Thus, we only need $O(M^2d)$ time to compute $\bK\bb$ for any $\bb \in \R^{Md}$.
In practice, we only need to store samples for computing $\bK\bb$ instead of constructing the whole kernel matrix. 
This reduces the memory usage from $O(M^2d^2)$ to $O(M^2d)$.
\end{example}

We note that the same idea in \Cref{example:iter} can be used to accelerate the KEF estimator if we adopt the conjugate gradient methods~\citep{van1983matrix} to solve \eqref{eq:tik-inverse},
because we have shown that the KEF estimator is equivalent to our Tikhonov regularized estimators with curl-free kernels.
As we shall see in experiments, this method is extremely fast in high dimensions.

\section{Theoretical Properties}
\label{sec:theory}
\begin{table}[t]\vspace{-.2cm}
	\centering
	\caption{Existing nonparametric score estimators, their kernel types, and regularization schemes. $\phi$ is from $k(\bx,\by) = \phi(\bx - \by)$.} 
	\label{tab:algorithms}
    \footnotesize
    \begin{center}
	\begin{tabular}{lcc} \toprule
        \sc Algorithm  & \sc Kernel & \sc Regularizer \\ \midrule
        SSGE & $k(\bx,\by)\bI_d$ & $\bone_{\{\sigma \geq \lambda\}}\sigma^{-1}$  \\
        Stein & $k(\bx,\by)\bI_d$ & $\bone_{\{\sigma > 0\}}(\lambda + \sigma)^{-1}$  \\
        KEF & $-\nabla^2\phi(\bx - \by)$ & $(\lambda + \sigma)^{-1}$ \\ 
        NKEF & $-\nabla^2\phi(\bx - \by)$ & $\bone_{\{\sigma > 0\}}(\lambda + \sigma)^{-1}$  \\
         \bottomrule
	\end{tabular}%
    \end{center}
    \vskip -0.1in
\end{table}

In this section, we give a general theorem on the convergence rate of score estimators in our framework, which provides a tighter error bound of SSGE~\citep{shi2018spectral}. We also investigate the case where samples are corrupted by a small set of points, and provide the convergence rate of the heuristic out-of-sample extension proposed in \citet{li2018gradient}. Proofs and assumptions are deferred to \cref{appendix:error-bounds}.

First, we follow \citet{bauer2007regularization,baldassarre2012multi} to characterize the regularizer.

\begin{definition}[\citet{bauer2007regularization}]
    \label{def:regularizer}
    We say a family of functions $g_\lambda: [0, \kappa^2] \to \R$, $0 < \lambda \leq \kappa^2$ is a \textit{regularizer} if there are constants $B, D, \gamma$ such that
     $\sup_{0 < \sigma \leq\kappa^2} | \sigma g_\lambda(\sigma)| \leq D$,
      $\sup_{0 < \sigma\leq \kappa^2} | g_\lambda(\sigma)| \leq B/\lambda$ 
      and $\sup_{0 < \sigma\leq \kappa^2} |1- \sigma g_\lambda(\sigma)| \leq \gamma$. 
      The \textit{qualification} of $g_\lambda$ is the maximal $r$ such that 
      $\sup_{0 < \sigma\leq \kappa^2} | 1- \sigma g_\lambda(\sigma)|\sigma^r \leq \gamma_r \lambda^r$,
       where $\gamma_r$ does not depend on $\lambda$.
\end{definition}

Now, we can use the idea of \citet[Theorem 10]{bauer2007regularization} to obtain an error bound of our estimator.
\begin{theorem} 
    \label{thm:error-bound}
    Assume Assumptions~\ref{assumption:domain}-\ref{assumption:bounded-trace} hold.
    Let $\bar r$ be the qualification of the regularizer $g_\lambda$, and $\hat \bs^g_{p,\lambda}$ be defined as in \refeq{eqn:general-score-estimator}. 
    Suppose there exists $f_0 \in \cH_\cK$ such that $\bs_p = L_\cK^r f_0$ for some $r \in [0, \bar r]$. Then we have for $\lambda = M^{-\frac{1}{2r+2}}$, 
    \[ \normz{ \hat \bs^g_{p,\lambda} - \bs_p }_{\cH_\cK} = O_p \left (M^{-\frac{r}{2r+2}} \right)  ,\]
    and for $r \in [0, \bar r - 1/2]$, we have
    \[ \llnormz{ \hat \bs^g_{p,\lambda} - \bs_p } = O_p \left (M^{-\frac{r + 1/2}{2r+2}} \right)  ,\]
    where $O_p$ is the Big-O notation in probability.
\end{theorem}

Note the qualification impacts the maximal convergence rate. 
As the qualification of Tikhonov regularization is $1$, from the error bound, we observe the well-known saturation phenomenon of Tikhonov regularization~\citep{engl1996regularization}, i.e., the convergence rate does not improve even if $\bs_p = L_\cK^r f_0$ for $r > 1$.
To alleviate this, we can choose the regularizer with a larger qualification.
For example, the spectral cut-off regularization and the Landweber iteration have qualification $\infty$, and the $\nu$-method has qualification $\nu$. This suggests that the $\nu$-method is appealing as it has a smaller iteration number than the Landweber iteration and a better maximal convergence rate than the Tikhonov regularization.

\begin{remark}[Stein]
    The consistency and convergence rate of the Stein estimator and its out-of-sample extension suggested in Example~\ref{example:stein} follow from Theorem~\ref{thm:error-bound}.
    The rate in $\llnorm{\cdot}$ is $O_p\left ( M^{-\theta_1} \right )$, where $\theta_1 \in [1/4, 1/3]$.
    The convergence rate of the original out-of-sample extension in \citet{li2018gradient} will be given in Corollary~\ref{cor:stein-bound}.
\end{remark}

\begin{remark}[SSGE]
    From Theorem~\ref{thm:error-bound}, the convergence rate in $\llnorm{\cdot}$ of SSGE is $O_p(M^{-\theta_2})$, where $\theta_2 \in [1/4, 1/2)$, which improves~\citet[Theorem 2]{shi2018spectral}. To see this, we assume the eigenvalues of $L_\cK$ are $\mu_1 > \mu_2 > \cdots$ and they decay as $\mu_J = O(J^{-\beta})$. The error bound provided by \citet{shi2018spectral} is
    \[ \llnormz{ \hat \bs_{p,\lambda} - \bs_p }^2 = O_p \left (\frac{J^2}{\mu_J (\mu_J - \mu_{J+1})^2 M} + \mu_J  \right)   .\]
    We can choose $J = M^{\frac{1}{4(\beta+1)}}$ to obtain $\llnormz{ \hat \bs_{p,\lambda} - \bs_p } = O_p(M^{-\frac{\beta}{8(\beta+1)}})$. The convergence rate is slower than $O_p(M^{-1/4})$, the worst case of Theorem~\ref{thm:error-bound}.
\end{remark}

\begin{remark}[KEF]
    Compared with Theorem 7(ii) in \citet{sriperumbudur2017density}, where they bound the Fisher divergence, 
    which is the square of the $L^2$-norm in our Theorem~\ref{thm:error-bound}, 
    we see that the two results are exactly the same. 
    The rate in this norm is $O_p\left ( M^{-\theta_3} \right )$, where $\theta_3 \in [1/4, 1/3]$.
\end{remark}

Next, we consider the case where estimators are not obtained from i.i.d. samples. Specifically, we consider how the convergence rate is affected when our data is the mixture of a set of i.i.d. samples and a set of fixed points.

\begin{theorem}
    \label{thm:mixture-error-bound}
    Under the same assumption of Theorem~\ref{thm:error-bound}, we define $g_\lambda(\sigma) := (\lambda + \sigma)^{-1}$, and choose $\bZ := \{\bz^n\}_{n\in[N]} \subseteq \cX$. Let $\bY := \{\by^m\}_{m\in[M]}$ be a set of i.i.d. samples drawn from $\rho$, and $\hat\bs_{p,\lambda,\bZ}$ be defined as in \refeq{eqn:general-score-estimator} with $\bX = \bZ \cup \bY$. Suppose $N = O(M^\alpha)$, then we have for $\lambda = M^{-\frac{1}{2r+2}}$, 
    \[ \sup_\bZ \normz{ \hat \bs_{p,\lambda, \bZ} - \bs_p }_{\cH_\cK} = O_p \left (M^{-\frac{r}{2r+2}} \right) + O(M^{\alpha - \frac{r}{r+1}}), \]
    where the $\sup_\bZ$ is taken over all $\{ \bz^n \}_{n\in[N]} \subseteq \cX$.
\end{theorem}
\begin{proof}[Proof Outline]
    Define $T_\bZ := \frac{1}{N}S^*_\bZ S_\bZ$, where $S_\bZ$ is the sampling operator. Let $\hat \bs_{p,\lambda}$ be defined as in \refeq{eqn:general-score-estimator} with $\bX = \bY$. We can write the estimator as $\hat \bs_{p,\lambda,\bZ} := g_\lambda(\hat L_\cK + R_\bZ)(\hat L_\cK + R_\bZ) \bs_p$, where $R_\bZ := \frac{N}{M+N}(T_\bZ - \hat L_\cK)$, and bound $\normz{\hat\bs_{p,\lambda,\bZ} - \hat\bs_{p,\lambda}}$ by
    $\normz{(g_\lambda(\hat L_\cK + R_\bZ) - g_\lambda(\hat L_\cK))\hat L_\cK \bs_p} + \normz{g_\lambda(\hat L_\cK + R_\bZ)R_\bZ\bs_p}$.
    It can be shown that the first term is $O\left (NM^{-1}\lambda^{-2} \right)$, and the second term is $O\left ( NM^{-1}\lambda^{-1} \right)$. 
    Combining these with Theorem~\ref{thm:error-bound}, we finish the proof.
\end{proof}
From Theorem~\ref{thm:mixture-error-bound}, we see that the convergence rate is not affected when data is corrupted by at most $O(M^{\frac{r}{2r+2}})$ points.
Under the same notation of this theorem, the out-of-sample extension of the Stein estimator proposed in \citet{li2018gradient} can be written as $\hat\bs_{p,\lambda,\bx}(\bx)$, which %
corrupts the i.i.d. data by a single test point. 
Then we can obtain the following bound for this estimator.
\begin{corollary}
    With the same assumptions and notations of Theorem~\ref{thm:mixture-error-bound}, we have
    \[ \sup_{\bx\in\cX} \norm{\hat\bs_{p,\lambda,\bx}(\bx) - \bs_p(\bx)}_2 = O_p(M^{-\frac{r}{2r+2}}). \]
    \label{cor:stein-bound}
\vspace{-5mm}
\end{corollary}

\section{Experiments}
\label{sec:experiments}

We evaluate our estimators on both synthetic and real data. In Sec.~\ref{sec:toy-experiments}, we consider a challenging grid distribution as described in the experiment of \citet{sutherland2017efficient} to test the accuracy of nonparametric score estimators in high dimensions and out-of-sample points, 
In Sec.~\ref{sec:exp-wae} we train Wasserstein autoencoders (WAE) with score estimation and compare the accuracy and the efficiency of different estimators. %
We mainly compare the following score estimators\footnote{Code is available at \url{https://github.com/miskcoo/kscore}.}:

\textbf{Existing nonparametric estimators}: Stein~\citep{li2018gradient}, SSGE~\citep{shi2018spectral}, KEF~\citep{sriperumbudur2017density}, and its low rank approximation NKEF$_\alpha$~\citep{sutherland2017efficient}, where $\alpha$ represents to use $\alpha M / 10$ samples.

\textbf{Parametric estimators}: In the WAE experiment, we also consider the sliced score matching (SSM) estimator~\citep{song2019sliced}, which is a parametric method and requires amortized training.

\textbf{Proposed}: The iterative curl-free estimator with the $\nu$-method, and the conjugate gradient version of the KEF estimator (KEF-CG), both described in Sec.~\ref{sec:scalability}.

\subsection{Synthetic Distributions} \label{sec:toy-experiments}

\begin{figure}[t]
	\centering
	\begin{subfigure}{0.49\linewidth}
		\centering
		\includegraphics[width=\linewidth]{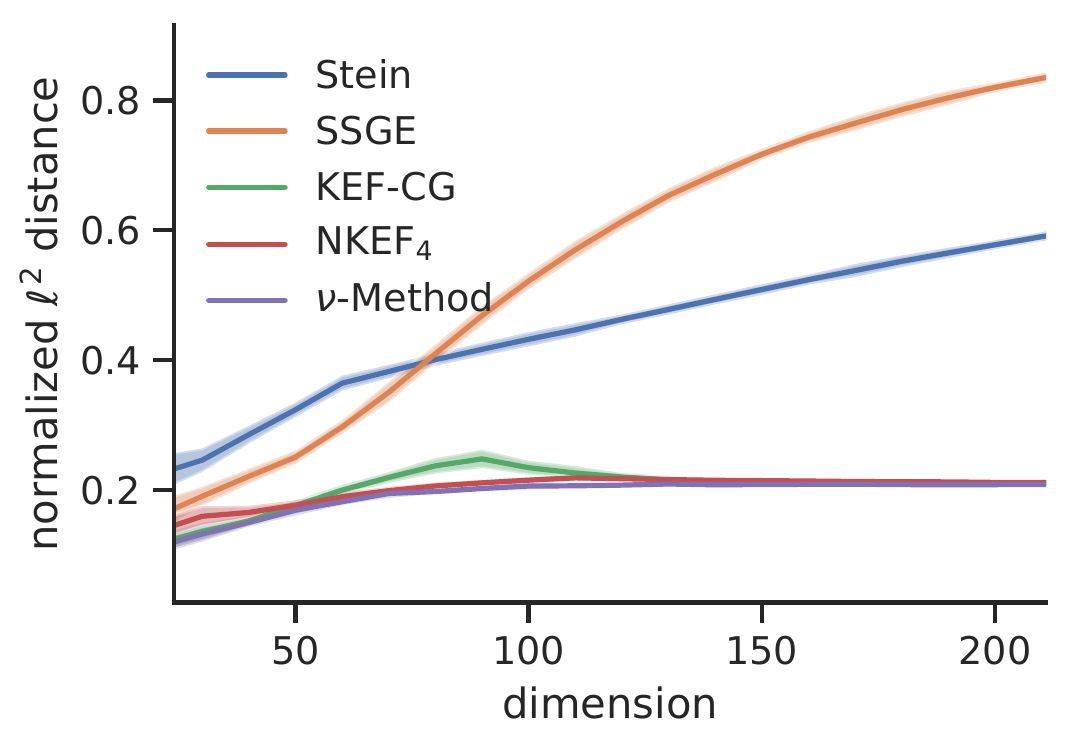}
		\caption{$M=128$}
	\end{subfigure} \hfill
	\begin{subfigure}{0.49\linewidth}
		\centering
		\includegraphics[width=\linewidth]{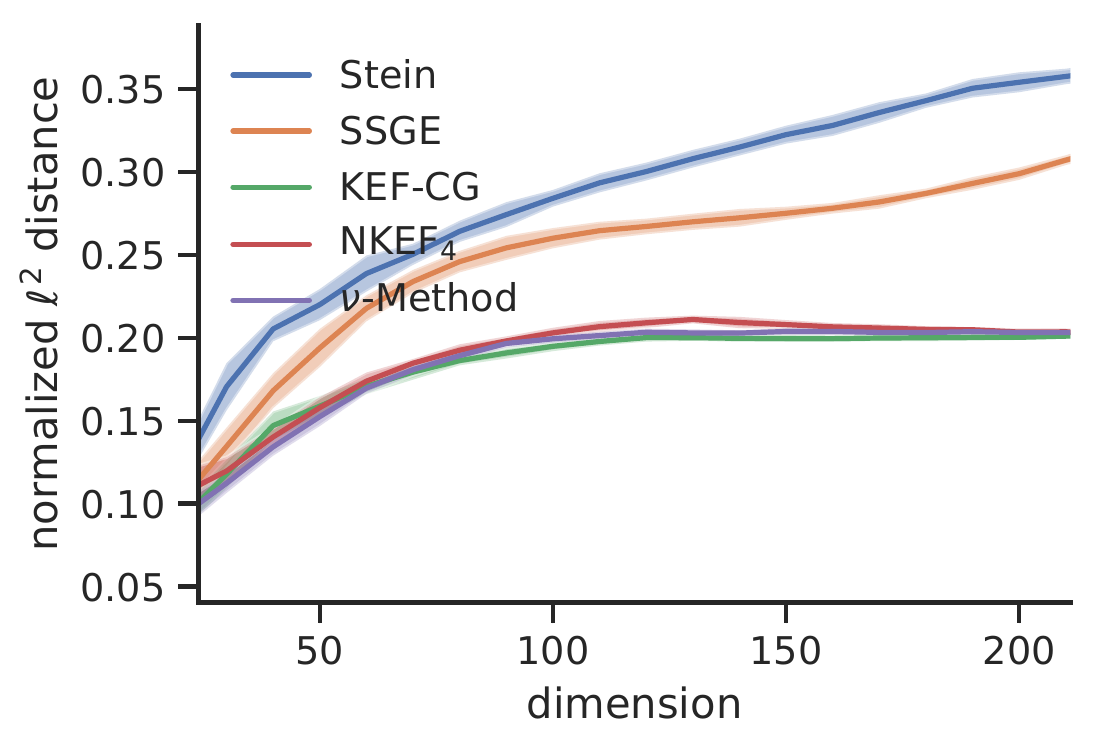}
		\caption{$M=512$}
	\end{subfigure} \\ 
	\begin{subfigure}{0.49\linewidth}
		\centering
		\includegraphics[width=\linewidth]{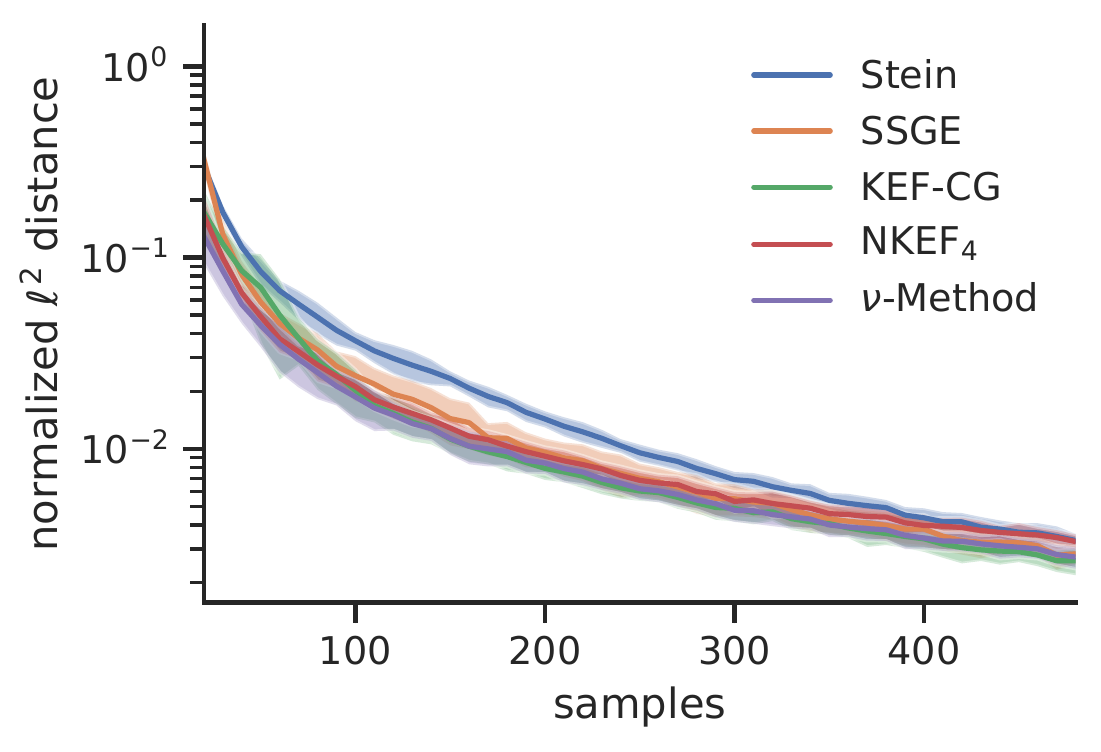}
		\caption{$d=16$}
	\end{subfigure} \hfill
	\begin{subfigure}{0.49\linewidth}
		\centering
		\includegraphics[width=\linewidth]{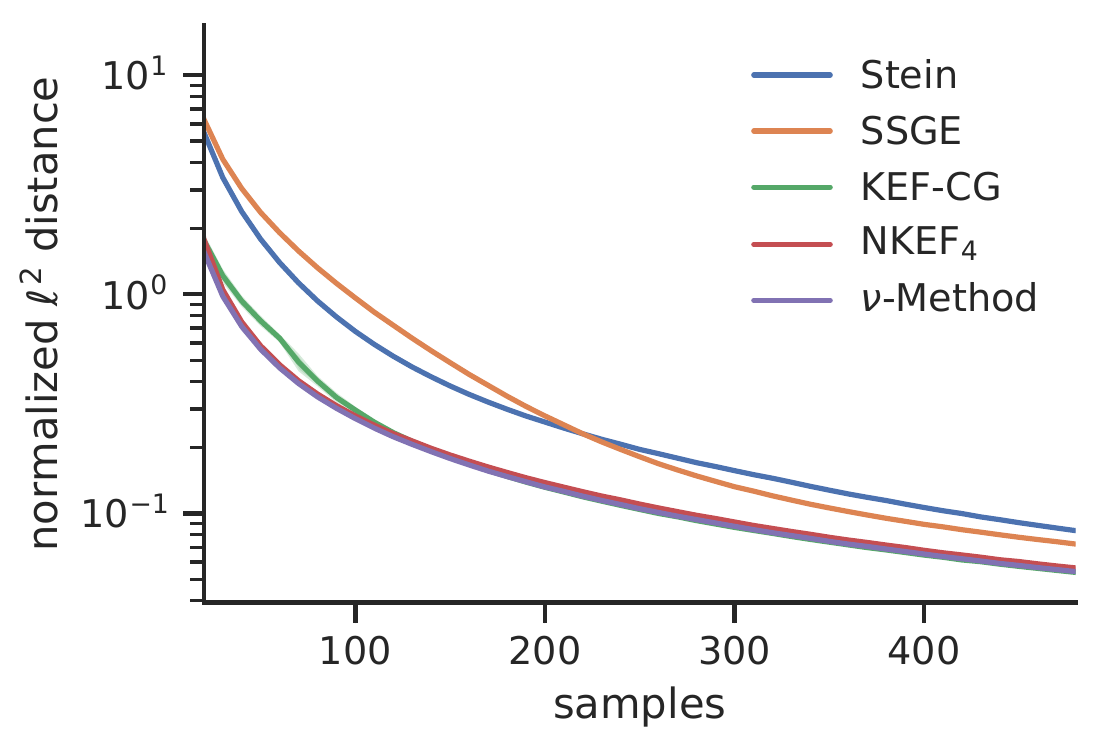}
		\caption{$d=128$}
	\end{subfigure} \vspace{-.2cm}
	\caption{Normalized distance $\E[\norm{\bs_p - \hat \bs_{p,\lambda}}_2^2] / d$ on grid data. In the first row, $M$ is fixed and $d$ varies. In the second row, $d$ is fixed and $M$ varies. Shaded areas are  three times the standard deviation.}
	\label{fig:grid-data}
\end{figure}

\begin{figure}[t]
    \centering
	\begin{subfigure}{.5\linewidth}
		\centering
		\includegraphics[width=\linewidth]{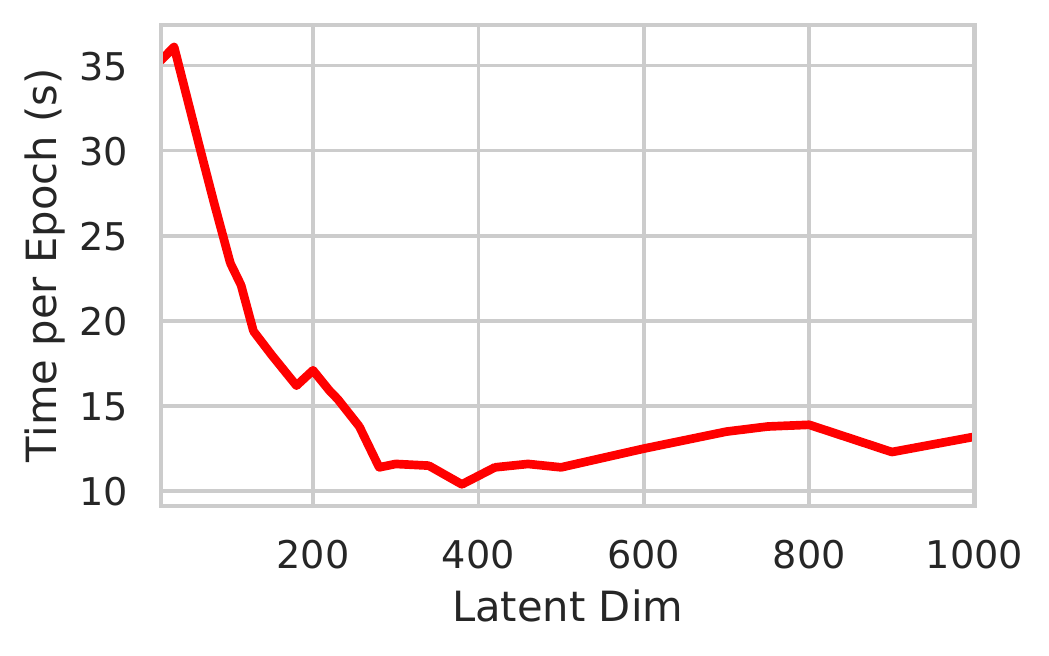}
		\caption{Computational Cost}
		\label{fig:matspectral-time}
	\end{subfigure}%
	\begin{subfigure}{.5\linewidth}
		\centering
		\includegraphics[width=\linewidth]{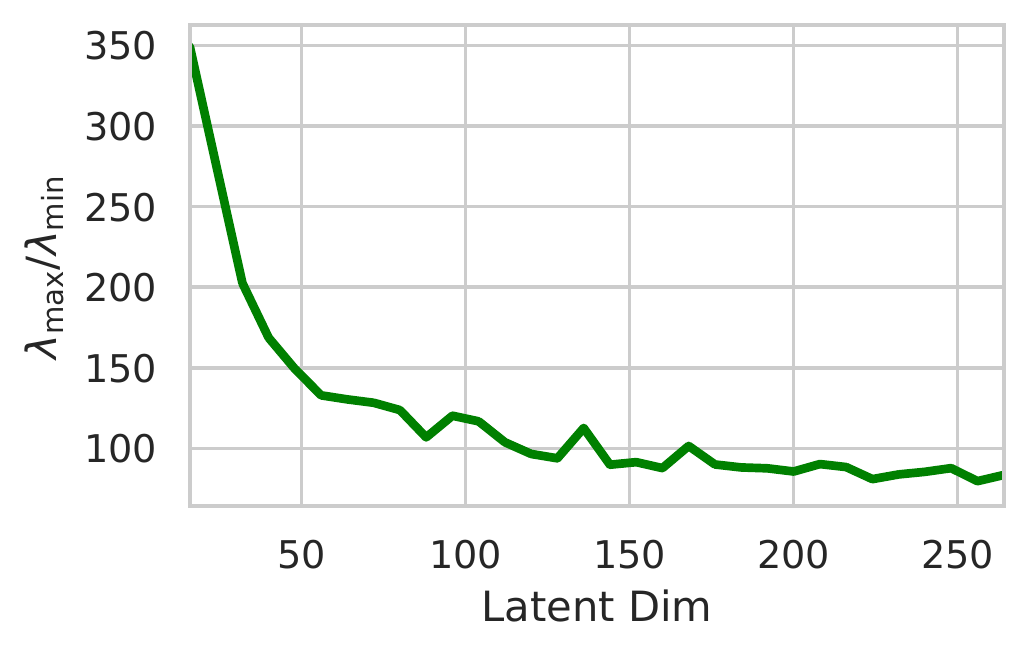}
		\caption{$\lambda_\textrm{max}/\lambda_\textrm{min}$}
		\label{fig:matspectral-eigen}
	\end{subfigure}\vspace{-.2cm}
	\caption{(a) Computational costs of KEF-CG for $\lambda = 10^{-5}$ on MNIST; (b) The ratio of the maximum and the minimum eigenvalues of kernel matrices.} \vspace{-.2cm}
\end{figure}

We follow \citet[Sec. 5.1]{sutherland2017efficient} to construct a $d$-dimensional grid distribution. It is the mixture of $d$ standard Gaussian distributions centered at $d$ fixed vertices in the unit hypercube. We change $d$ and $M$ respectively to test the accuracy and the convergence of score estimators, and use $1024$ samples from the grid distribution to evaluate the $\ell^2$. %
We report the result of $32$ runs in \reffig{fig:grid-data}. 

We can see that the effect of hypothesis space is significant. The diagonal kernels used in SSGE and Stein degrade the accuracy in high dimensions, while curl-free kernels provide better performance. In low dimensions, all estimators are comparable, and the computational cost of diagonal kernels is lower than that of curl-free kernels. This suggests favoring the diagonal kernels in low dimensions.
Possibly because this dataset does not make the convergence rate saturate, we find different regularization schemes produce similar results.
The iterative score estimator based on the $\nu$-method is among the best and KEF-CG closely tracked them even with large $d$ and $M$.

\begin{table}[t]\vspace{-.2cm}
	\centering
	\caption{Negative log-likelihoods on MNIST datasets and per epoch time on 128 latent dimension. All models are timed on GeForce GTX TITAN X GPU.}
	\label{tab:nll-mnist}\vspace{-.2cm}
    \scriptsize
    \begin{center}
    \begin{small}
    \begin{sc}
    \resizebox{\columnwidth}{!}{%
	\begin{tabular}{lccccc} \toprule
		Latent Dim& 8 & 32 & 64 & 128 & Time \\ \midrule
		Stein & 97.15 & 92.10 & 101.60 & 114.41 & 4.2s \\
		SSGE & 97.24 & 92.24 & 101.92 & 114.57 & 9.2s \\
        KEF & 97.07 & 90.93 & 91.58 & 92.40 & 201.1s \\
        NKEF$_2$ & 97.71 & 92.29 & 92.82 & 94.14 &  36.4s \\
        NKEF$_4$ & 97.59 & 91.19 & 91.80 & 92.94 &  97.5s \\
        NKEF$_8$ & 97.23 & 90.86 & 92.39 & 92.49 &  301.2s \\
        \midrule
        KEF-CG & 97.39 & 90.77 & 92.66 & \textbf{92.05} & 13.7s \\
        $\nu$-method & 97.28 & 90.94 & \textbf{91.48} & 92.10 & 78.1s \\
        \midrule
        SSM & \textbf{96.98} & \textbf{89.06} & 93.06 & 96.92 & 6.0s \\
         \bottomrule
	\end{tabular}%
	}%
    \end{sc}
    \end{small}
    \end{center}
    \vskip -0.2in
\end{table}

\subsection{Wasserstein Autoencoders} \label{sec:exp-wae}

Wasserstein autoencoder (WAE) \citep{tolstikhin2017wasserstein} is a latent variable model $p(\bz, \bx)$ with observed and latent variables $\bx \in \cX$ and $\bz \in \cZ$, respectively. $p(\bz, \bx)$ is defined by a prior $p(\bz)$ and a distribution of $\bz$ conditioned on $\bx$, and can be written as $p(\bz, \bx) = p(\bz)p_\theta(\bx | \bz)$. WAEs aim at minimizing Wasserstein distance $\cW_c(p_X, p_G)$ between the data distribution $p_X(\bx)$ and $p_G(\bx) := \int p(\bz, \bx)d\bz$, where $c$ is a metric on $\cX$. \citet{tolstikhin2017wasserstein} showed that when $p_\theta(\bx | \bz)$ maps $\bz$ to $\bx$ deterministically by a function $G: \cZ \to \cX$, it suffices to minimize %
$	\E_{p_X(\bx)}\E_{q_\phi(\bz|\bx)}[\normz{\bx - G(\bz)}_2^2] + \lambda \cD(q_\phi(\bz), p(\bz))$, 
where $\cD$ is a divergence of two distributions and $q_\phi(\bz | \bx)$ is a parametric approximation of the posterior. %
When we choose $\cD$ to be the KL divergence, the entropy term of $q_\phi(\bz) := \int q_\phi(\bz |\bx)p_X(\bx)d\bx$ in the loss function is intractable~\citep{song2019sliced}. 
If $\bz$ can be parameterized by $f_\phi(\bx)$ with $\bx \sim p_X$, the gradient of the entropy can be estimated using score estimators as 
$\E_{p_X(\bx)}[\nabla_\bz \log q_\phi(\bz)\nabla_\phi f_\phi(\bx)]$~\citep{shi2018spectral,song2019sliced}.

We train WAEs on MNIST and CelebA and repeat each configuration $3$ times. The average negative log-likelihoods for MNIST estimated by AIS \citep{neal2001annealed} are reported in \reftab{tab:nll-mnist}. The results for CelebA are reported in \cref{appendix:experiments}.
We can see that the performance of these estimators is close in low latent dimensions, and the parametric method is slightly better than nonparametric ones as we have continuously generated samples.
However, in high dimensions, estimators based on curl-free kernels significantly outperform those based on diagonal kernels and parametric methods. 
This is probably due to guarantee that the estimates at all locations form a gradient field.

As discussed in Sec.~\ref{sec:scalability}, curl-free kernels are computationally expensive. 
This is shown in \reftab{tab:nll-mnist} by the running time of the original KEF algorithm.
By comparing the time and performance of NKEF$_\alpha$ with $\alpha=2,4,8$, we see that in order to get meaningful speed-up in high dimensions,
low-rank approximation methods have to sacrifice the performance, 
which are outperformed by the iterative curl-free estimators based on the $\nu$-method.
KEF-CG is the fastest curl-free method in high dimensions while the performance is comparable with the original KEF. 
\reffig{fig:matspectral-time} shows the training time of KEF-CG in different latent dimensions. Surprisingly, the speed rapidly increases with increasing latent dimension and then flattens out. The convergence rate of conjugate gradient is determined by the condition number, which means the kernel matrix $\bK$ becomes well-conditioned in high dimensions~(see \reffig{fig:matspectral-eigen}). %

We found with large $d$, SSGE required at least $97\%$ eigenvalues to attain the reported likelihood.
We also ran SSGE with curl-free kernels and found only $13\%$ eigenvalues are required to attain a comparable result when $d = 8$. From these observations, a possible reason why diagonal kernels degrade the performance in high dimensions is that the distribution is complicated while the hypothesis set is simple, so the small number of eigenfunctions are insufficient to approximate the target. This can also be observed from \reffig{fig:grid-data}, where the performance of diagonal kernels and curl-free kernels are closer as $M$ increases since more eigenfunctions are provided.

\section{Conclusion}
\label{sec:conclusion}

Our contributions are two folds. Theoretically, we present a unifying view of nonparametric score estimators, %
and clarify the relationships of existing estimators. 
Under this perspective, we provide a unified convergence analysis of existing estimators, which improves existing error bounds. 
Practically, we propose an iterative curl-free estimator with nice theoretical properties and computational benefits, and develop a fast conjugate gradient solver for the KEF estimator.

\section*{Acknowledgements}

We thank Ziyu Wang for reading an earlier version of this manuscript and providing valuable feedback.
This work was supported by the National Key Research and Development Program of China (No. 2017YFA0700904), NSFC Projects (Nos. 61620106010, U19B2034, U1811461), Beijing NSF Project (No. L172037), Beijing Academy of Artificial Intelligence (BAAI), Tsinghua-Huawei Joint Research Program, Tsinghua Institute for Guo Qiang, and the NVIDIA NVAIL Program with GPU/DGX Acceleration. JS was also supported by a Microsoft Research Asia Fellowship.

\bibliography{references}
\bibliographystyle{icml2020}

\clearpage
\appendix
\onecolumn

In \cref{appendix:experiments} we provide additional details and further results of experiments. In \cref{appendix:error-bounds}, we list the assumptions we used, and prove the non-asymptotic verison of Theorem~\ref{thm:error-bound} and \ref{thm:mixture-error-bound}. In \cref{appendix:details-in-sec3}, we give the details of Sec.~\ref{sec:estimators}, including deriving algorithms presented in Sec.~\ref{sec:regularization-schemes}, examples in Sec.~\ref{sec:connection} and a general formula for curl-free kernels. \Cref{appendix:technique-results} includes some technical results used in proofs. Finally, We present samples drawn from trained WAEs in \cref{appendix:samples}.

\section{Experiment Details and Additional Results} \label{appendix:experiments}
In experiments, we use the IMQ kernel $k(\bx, \by) := (1 + \normz{\bx-\by}_2^2/\sigma^2)^{-1/2}$ and its curl-free version in corresponding kernel estimators. We use the median of the pairwise Euclidean distances between samples as the kernel bandwidth. The parameter $\nu$ of the $\nu$-method is set to $1$.
The maximum iteration number of KEF-CG is 40 and the convergence tolerance of it is $10^{-4}$. 

\subsection{Grid Distributions}
We use $\alpha M$ eigenvalues in SSGE with $\alpha$ searched in $\{ 0.99, 0.97, 0.95, 0.9, 0.8, 0.7, 0.6, 0.5, 0.4 \}$. We search the number of iteration steps of the $\nu$-method in $\{ 20, 30, 40, 50, 60, 70, 80, 90, 100 \}$. We search the regularization coefficient $\lambda$ of Stein, NKEF, KEF-CG in $\{ 10^{-k} : k = 0, 1, \cdots, 8\}$. The experiments are repeated $32$ times.

\subsection{Wasserstein Autoencoders}
We use the standard Gaussian distribution $\cN(0, I)$ as the prior $p(\bz)$, and $\cN(\mu_\phi(\bx), \sigma_\phi^2(\bx))$ as the approximated posterior $q_\phi(\bz | \bx)$, and $\mathrm{Bernoulli}(G_\theta(\bz))$ as the generator $p_\theta(\bx | \bz)$.
We use minibatch size 64. Models are optimized by the Adam optimizer with learning rate $10^{-4}$. Each configuration is repeated $3$ times, and the mean and the standard deviation are reported in \reftab{appendix:table/ll-mnist} and \reftab{appendix:table/fid-celeba}. All models are timed on GeForce GTX TITAN X GPU.

\begin{table}[ht]
	\centering
	\caption{Negative log-likelihoods on the MNIST dataset and per epoch time on 128 latent dimension.}
    \label{appendix:table/ll-mnist}
    \begin{center}
    \begin{small}
    \begin{sc}
    \vskip 0.05in
	\begin{tabular}{lccccc} \toprule
		Latent Dim& 8 & 32 & 64 & 128 & Time \\ \midrule
		Stein & 97.15 $\pm$ 0.14 & 92.10 $\pm$ 0.07 & 101.60 $\pm$ 0.44 & 114.41 $\pm$ 0.25 & 4.2s \\
		SSGE & 97.24 $\pm$ 0.07 & 92.24 $\pm$ 0.17 & 101.92 $\pm$ 0.08 & 114.57 $\pm$ 0.23 & 9.2s \\
        KEF & 97.07 $\pm$ 0.03 & 90.93 $\pm$ 0.23 & 91.58 $\pm$ 0.03 & 92.40 $\pm$ 0.34 & 201.1s \\
        NKEF$_2$ & 97.71 $\pm$ 0.24 & 92.29 $\pm$ 0.41 & 92.82 $\pm$ 0.18 & 94.14 $\pm$ 0.69 &  36.4s \\
        NKEF$_4$ & 97.59 $\pm$ 0.15 & 91.19 $\pm$ 0.08 & 91.80 $\pm$ 0.12 & 92.94 $\pm$ 0.58 &  97.5s \\
        NKEF$_8$ & 97.23 $\pm$ 0.06 & 90.86 $\pm$ 0.09 & 92.39 $\pm$ 1.32 & 92.49 $\pm$ 0.41 &  301.2s \\
        \midrule
        KEF-CG & 97.39 $\pm$ 0.22 & 90.77 $\pm$ 0.12 & 92.66 $\pm$ 0.67 & \textbf{92.05} $\pm$ 0.06 & 13.7s \\
        $\nu$-method & 97.28 $\pm$ 0.17 & 90.94 $\pm$ 0.02 & \textbf{91.48} $\pm$ 0.09 & 92.10 $\pm$ 0.06 & 78.1s \\
        \midrule
        SSM & \textbf{96.98} $\pm$ 0.27 & \textbf{89.06} $\pm$ 0.01 & 93.06 $\pm$ 0.68 & 96.92 $\pm$ 0.08 & 6.0s \\
         \bottomrule
	\end{tabular}%
    \end{sc}
    \end{small}
    \end{center}
    \vskip -0.1in
\end{table}
\begin{table}[ht]
    \centering
    \caption{Fr\'echet Inception Distances on the CelebA dataset and per epoch time on 128 latent dimension.}
    \label{appendix:table/fid-celeba}
    \scriptsize
    \begin{center}
    \begin{small}
    \begin{sc}
    \vskip 0.05in
	\begin{tabular}{lccccc} \toprule
		Latent Dim& 8 & 32 & 64 & 128 & Time \\ \midrule
		Stein & 73.85 $\pm$ 1.39 & 58.29 $\pm$ 0.46 & 57.54 $\pm$ 0.57 & 76.31 $\pm$ 1.33 & 164.4s \\
		SSGE & 72.49 $\pm$ 1.09 & 58.01 $\pm$ 0.60 & 58.39 $\pm$ 1.00 & 76.85 $\pm$ 1.12 & 172.2s \\
		NKEF$_2$ & 75.12 $\pm$ 1.55 & 53.92 $\pm$ 0.29 & 51.16 $\pm$ 0.30 & 55.17 $\pm$ 0.43 & 244.7s \\
		NKEF$_4$ & 73.15 $\pm$ 0.77 & 54.54 $\pm$ 1.02 & 50.76 $\pm$ 0.19 & 53.70 $\pm$ 0.10 & 412.5s \\
        \midrule
		KEF-CG & 72.92 $\pm$ 0.60 & 54.32 $\pm$ 0.31 & 50.44 $\pm$ 0.20 & \textbf{50.66} $\pm$ 0.89 & 166.2s \\
		$\nu$-method & 72.02 $\pm$ 1.22 & 52.86 $\pm$ 0.20 & \textbf{50.16} $\pm$ 0.23 & 52.80 $\pm$ 0.43 & 220.9s \\
        \midrule
		SSM & \textbf{69.72} $\pm$ 0.25 & \textbf{49.93} $\pm$ 0.74 & 72.68 $\pm$ 1.75 & 94.07 $\pm$ 3.57 & 163.3s \\
         \bottomrule
	\end{tabular}%
    \end{sc}
    \end{small}
    \end{center}
    \vskip -0.1in
\end{table}

\paragraph{MNIST}
We parameterize $\mu_\phi$, $\sigma^2_\phi$ and $G_\theta(\bz)$ by fully-connected neural networks with two hidden layers, both of which consist of 256 units activated by \texttt{ReLU}. For SSM, the score is parameterized by a fully-connected neural network with two hidden layers consisting of 256 units activated by \texttt{tanh}. 
The regularization coefficients of Stein, KEF, NKEF, KEF-CG are searched in $\{ 10^{-k} : k = 2, 3, \cdots, 7 \}$ for the best log-likelihood, and the number of iteration steps of the $\nu$-method are searched in $\{ 50, 70, \cdots, 150 \}$, and we use $\alpha M$ eigenvalues in SSGE with $\alpha$ searched in $\{ 0.99, 0.97, 0.95, 0.93, 0.91, 0.89, 0.87 \}$. We run 1000 epoches and evaluate the model by AIS~\citep{neal2001annealed}, where the parameters are the same as in SSGE. Specifically, we set the step size of HMC to $10^{-6}$, and the leapfrog step to $10$. We use $5$ chains and set the temperature to $10^3$.

\paragraph{CelebA}
We parameterize $\mu_\phi$, $G_\theta$ by convolutional neural networks similar to~\citet{song2019sliced}. $\sigma_\phi^2$ is set to $1$. For SSM, we use the same network as in MNIST to parameterize the score.
The regularization coefficients of Stein, KEF, NKEF, KEF-CG are searched in $\{ 10^{-k} : k = 2, 3, \cdots, 7 \}$ for the best log-likelihood, and the number of iteration steps of the $\nu$-method are searched in $\{ 20, 30, 40, 50, 60, 70 \}$, and we use $\alpha M$ eigenvalues in SSGE with $\alpha$ searched in $\{ 0.99, 0.97, 0.95, 0.93, 0.91, 0.89, 0.87 \}$. We run 100 epoches and evaluate the model using the Fr\'echet Inception Distance~(FID). As KEF and NKEF$_8$ are slow, we do not compare them in this dataset. Results are reported in \reftab{appendix:table/fid-celeba}.

\section{Error Bounds} \label{appendix:error-bounds}

In the following, we suppress the dependence of $\cH_\cK$ on $\cK$ for simplicity. We use $\hsnorm{\cdot}$ to denote the Hilbert-Schmidt norm of operators. The assumptions required in obtaining an error bound are listed below.

\begin{assumption}
    \label{assumption:domain}
    $\cX$ is a non-empty open subset of $\R^d$, with piecewise $C^1$ boundary.
\end{assumption}
\begin{assumption}
    \label{assumption:boundary-extension}
    $p$, $\log p$ and each element of $\cK$ are continuously differentiable.
    $p$ and its total derivative $Dp: \cX \to \R^d$ can both be continuously extended to $\bar \cX$, where $\bar \cX$ is the closure of $\cX$. Each element of $\cK$ and its total derivative can be continuously extended to $\bar \cX \times \bar \cX$.
\end{assumption}
\begin{assumption}
    \label{assumption:integration-by-parts}
    For all $i, j \in [d]$, $\cK(\bx, \bx)_{ij} p(\bx) = 0$ on $\partial \cX$, and $\sqrt{|\cK(\bx, \bx)_{ij}|}p(\bx) = o(\normz{\bx}_2^{1-d})$ as $\bx \to \infty$, where $\partial \cX := \bar \cX \setminus \cX$.
\end{assumption}
\begin{assumption}
    \label{assumption:bounded-div}
    Define an $\cH_\cK$-valued random variable $\xi_\bx := \divgers{\bx}\cK_\bx^\trans$, let $\xi := \int_\cX \xi_\bx d\rho$. There are two constants $\Sigma, K$, such that 
    \[ \int_\cX \left \{ \exp\left (\frac{\normz{ \xi_\bx - \xi }_\cH}{K}\right) - \frac{\normz{ \xi_\bx - \xi }_\cH}{K} - 1 \right \} d\rho \leq \frac{\Sigma^2}{2K^2}. \]
\end{assumption}
\begin{assumption}
    \label{assumption:bounded-trace}
    There is a constant $\kappa > 0$ such that $\sup_{\bx \in \cX} \tr \cK(\bx,\bx) \leq \kappa^2$.
\end{assumption}

Assumptions \ref{assumption:domain}-\ref{assumption:integration-by-parts} are similar to those in~\citet{sriperumbudur2017density}. They guarantee the integration by parts is valid, so we can obtain $\E_\rho[ \cK_\bx \nabla \log p] = -\E_\rho[ \divgers{\bx}  \cK^\trans_\bx ]$. Assumptions \ref{assumption:bounded-div} and \ref{assumption:bounded-trace} come from \citet{bauer2007regularization}, and are used in the concentration inequalities. Note that Assumption~\ref{assumption:bounded-div} can be replaced by a stronger one that $\normz{\xi_\bx - \xi}_\cH$ is uniformly bounded on $\cX$. 

We follows the idea of \citet[Theorem 10]{bauer2007regularization} to prove Theorem~\ref{thm:error-bound}. The non-asymptotic version is given as follows
\begin{theorem}
    \label{thm:appendix-error-bound}
    Assume Assumptions~\ref{assumption:domain}-\ref{assumption:bounded-trace} hold.
    Let $\bar r$ be the qualification of the regularizer $g_\lambda$, and $\hat \bs^g_{p,\lambda}$ be defined as in \refeq{eqn:general-score-estimator}. Suppose there exists $f_0 \in \cH_\cK$ such that $\bs_p = L_\cK^r f_0$, for some $r \in [0, \bar r]$.
    Then for any $0 < \delta < 1$, $M \geq (2\sqrt 2 \kappa^2 \log(4/\delta))^{\frac{2r+2}{r}}$, choosing $\lambda = M^{-\frac{1}{2r+2}}$, the following inequalities hold with probability at least $1-\delta$
    \[ \normz{ \hat \bs_{p,\lambda} - \bs_p }_{\cH} \leq C_1 M^{-\frac{r}{2r+2}} \log \frac{4}{\delta}  ,\]
    and for $r \in [0, \bar r - 1/2]$, we have
    \[ \llnormz{ \hat \bs_{p,\lambda} - \bs_p } \leq C_2 M^{-\frac{2r + 1}{4r+4}} \log \frac{4}{\delta}  ,\]
    where $C_1 = 2B(K+\Sigma) + 2\sqrt 2 B\kappa^2\normz{\bs_p}_\cH + (\gamma_r +\kappa^2  \gamma c_r)\normz{f_0}_\cH$, and $C_2 = 2B(K+\Sigma)\kappa + 2\sqrt 2 B\kappa^3\normz{\bs_p}_\cH + ((\gamma_r + \kappa^2 \gamma_{\frac{1}{2}} c_r) + c_{\frac{1}{2}}(\gamma_r +\kappa^2  \gamma c_r))\normz{f_0}_\cH$, and $c_r$ is a constant depending on $r$. $O_p$ is the Big-O notation in probability.
\end{theorem}
\begin{proof}
    We consider the following decomposition
    \[ \begin{aligned}
        \normz{ \hat \bs_{p,\lambda} - \bs_p }_{\cH} 
        & \leq \normz{ g_\lambda(\hat L_\cK) (\hat \bzeta - \bzeta) }_{\cH} 
        + \normz{ g_\lambda(\hat L_\cK) L_\cK \bs_p - \bs_p }_{\cH} \\
        & \leq \normz{ g_\lambda(\hat L_\cK) (\hat \bzeta - \bzeta) }_{\cH} 
        + \normz{ g_\lambda(\hat L_\cK) (L_\cK - \hat L_\cK) \bs_p}_{\cH} 
        + \normz{ r_\lambda(\hat L_\cK) \bs_p }_{\cH},
    \end{aligned} \]
    where $r_\lambda(\sigma) := g_\lambda(\sigma)\sigma - 1$. By Definition~\ref{def:regularizer}, we have $\normz{g_\lambda(\hat L_\cK)} \leq B / \lambda$. From Lemma~\ref{lemma:xi-concentration} and \ref{lemma:L-concentration}, with probability at least $1 - \delta$, we have
    \[
        \normz{ g_\lambda(\hat L_\cK) (\hat \bzeta - \bzeta) }_{\cH}  
         +\normz{ g_\lambda(\hat L_\cK) (L_\cK - \hat L_\cK) \bs_p}_{\cH} 
        \leq \frac{2B(K + \Sigma) + 2\sqrt 2 B \kappa^2 \norm{\bs_p}_\cH}{\lambda\sqrt M} \log \frac{4}{\delta}.
     \]
     By Definition~\ref{def:regularizer}, $\normz{r_\lambda(\hat L_\cK)L_\cK^r} \leq \gamma_r \lambda^r$ and $\normz{r_\lambda(\hat L_\cK)} \leq \gamma$, then
     \[ \begin{aligned}
        \normz{r_\lambda(\hat L_\cK)\bs_p}_\cH &\leq \normz{r_\lambda(\hat L_\cK)\hat L_\cK^r f_0}_\cH + \normz{r_\lambda(\hat L_\cK)(L_\cK^r - \hat L_\cK^r) f_0}_\cH \\
        &\leq \gamma_r \lambda^r \norm{f_0}_\cH + \gamma \normz{L^r_\cK - \hat L^r_\cK}  \norm{f_0}_\cH.
     \end{aligned} \]
     When $r \in [0, 1]$, from \citet[Theorem 1]{bauer2007regularization}, there exists a constant $c_r$ such that $\normz{L_\cK^r - \hat L_\cK^r} \leq c_r \normz{L_\cK - \hat L_\cK}^r$. 
    Then by Lemma~\ref{lemma:L-concentration}, and choose $\lambda \geq 2\sqrt{2}\kappa^2 M^{-1/2} \log ( 4/\delta)$, we have
     \[ \normz{L_\cK^r - \hat L_\cK^r} \leq c_r \left (\frac{2\sqrt 2 \kappa^2}{\sqrt M} \log \frac{4}{\delta} \right )^r \leq c_r \lambda^r . \]
     Collecting the above results, 
     \[ \normz{\hat \bs_{p,\lambda} - \bs_p}_\cH 
     \leq  \left ( \frac{A_1}{\lambda\sqrt M} + A_2 \lambda^r \right ) \log \frac{4}{\delta}, \]
     where $A_1, A_2$ are constants which do not depend on $\lambda$ and $M$. Then, we can choose $\lambda = M^{-\frac{1}{2r+2}}$ to obtain the bound. Combining with $\lambda \geq 2\sqrt{2}\kappa^2 M^{-1/2} \log(4/\delta)$, we require $M^{\frac{r}{2r+2}} \geq 2\sqrt 2\kappa^2\log(4/\delta)$.

     When $r > 1$, from \Cref{lemma:operator-lipechitz}, there exists a constant $c_r^\prime$ such that $\hsnormz{L_\cK^r - \hat L_\cK^r} \leq c_r^\prime \hsnormz{L_\cK - \hat L_\cK}$. Then $\hsnormz{L_\cK^r - \hat L_\cK^r} \leq 2\sqrt 2 c_r^\prime \kappa^2 M^{-1/2} \log (4/\delta)$, and a similar discussion can be applied to obtain the bound.

     Note that $\llnormz{\hat \bs_{p,\lambda} - \bs_p} = \normz{\sqrt{L_\cK}(\hat \bs_{p,\lambda} - \bs_p)}_\cH$. Then we can apply the above discussion to obtain the bound for $\llnorm{\cdot}$.
\end{proof}

Next, we give the non-asymptotic version of Theorem~\ref{thm:mixture-error-bound} as follows
\begin{theorem}
    Under the same assumption of Theorem~\ref{thm:appendix-error-bound}, we define $g_\lambda(\sigma) := (\lambda + \sigma)^{-1}$, and choose $\bZ := \{\bz^n\}_{n\in[N]} \subseteq \cX$. Let $\bY := \{\by^m\}_{m\in[M]}$ be a set of i.i.d. samples drawn from $\rho$. Let $\hat\bs_{p,\lambda,\bZ}$ be defined as in \refeq{eqn:general-score-estimator} with $\bX = \bZ \cup \bY$. 
    Suppose $N = M^\alpha$, then for any $0 < \delta < 1$, $M \geq (2\sqrt 2 \kappa^2 \log(4/\delta))^{\frac{2r+2}{r}}$, choosing $\lambda = M^{-\frac{1}{2r+2}}$, the following inequalities hold with probability at least $1-\delta$
    \[ \sup_{\bZ} \normz{ \hat \bs_{p,\lambda, \bZ} - \bs_p }_{\cH} \leq C_1 M^{-\frac{r}{2r+2}} \log \frac{4}{\delta} + C_3 M^{\alpha - \frac{r}{r+1}} ,\]
    where $C_3 := 2(\kappa^2+1)^2\normz{\bs_p}_\cH$, and the $\sup_\bZ$ is taken over all $\{ \bz^n \}_{n\in[N]} \subset \cX$.

    In particular, when $\alpha = \frac{r}{2r+2}$, we have
    \[ \sup_{\bZ} \normz{ \hat \bs_{p,\lambda, \bZ} - \bs_p }_{\cH} \leq (C_1 + C_3) M^{-\frac{r}{2r+2}} \log \frac{4}{\delta}. \]
\end{theorem}

\begin{proof}
    We define $T_\bZ := \frac{1}{N}S^*_\bZ S_\bZ$, where $S_\bZ f := (f(\bz^1), \cdots, f(\bz^N))$ is the sampling operator. Let $\hat L_\cK := T_\bY$ and $\hat \bs_{p,\lambda}$ be the estimator obtained from $\bY$. Then we can write $\hat \bs_{p,\lambda,\bZ} := g_\lambda(\hat L_\cK + R_\bZ)(\hat L_\cK + R_\bZ) \bs_p$, where $R_\bZ := \frac{N}{M+N}(T_\bZ -\hat L_\cK )$. 
    
    We can bound the error as follows
    \[ \begin{aligned}
    \normz{ \hat \bs_{p,\lambda, \bZ} - \bs_p }_\cH
      & \leq  \normz{\hat\bs_{p,\lambda,\bZ} - \hat\bs_{p,\lambda}}_\cH 
      + \normz{\hat \bs_{p,\lambda} - \bs_p}_\cH \\
    & \leq \normz{(g_\lambda(\hat L_\cK + R_\bZ) - g_\lambda(\hat L_\cK))\hat L_\cK \bs_p}_\cH
     + \normz{g_\lambda(\hat L_\cK + R_\bZ)R_\bZ\bs_p}_\cH
     + \normz{\hat \bs_{p,\lambda} - \bs_p}_\cH.
    \end{aligned} \]
    The last term has been bounded by Theorem~\ref{thm:appendix-error-bound}, and we consider the first two terms. Since $g_\lambda(\sigma) = (\lambda + \sigma)^{-1}$ is Lipschitz in $[0, \infty)$, from \Cref{lemma:operator-lipechitz}, we have $\hsnormz{g_\lambda(\hat L_\cK + R_\bZ) - g_\lambda(\hat L_\cK)} \leq \hsnormz{R_\bZ} / \lambda^2$. Note $\hsnormz{g_\lambda(\hat L_\cK + R_\bZ) R_\bZ} \leq \hsnormz{R_\bZ} / \lambda$, we obtain
      \[ \begin{aligned}
        \normz{\hat\bs_{p,\lambda,\bZ} - \hat\bs_{p,\lambda}}_\cH 
      &\leq \left ( \frac{\kappa^2}{\lambda^2} + \frac{1}{\lambda} \right ) \hsnormz{R_\bZ} \normz{\bs_p}_\cH
      \leq \left ( \frac{\kappa^2}{\lambda^2} + \frac{1}{\lambda} \right ) \frac{2\kappa^2 N}{M + N} \normz{\bs_p}_\cH \\
      &\leq \frac{ 2(\kappa^2 + 1)^2N}{\lambda^2 M} \norm{\bs_p}_\cH
      =  2(\kappa^2 + 1)^2 M^{\alpha - \frac{r}{r+1}} \norm{\bs_p}_\cH.
      \end{aligned} \]
      Combining with Theorem~\ref{thm:appendix-error-bound}, and noticing that the right hand does not depend on $\bZ$, we obtain the final bound.
\end{proof}

Finally, we prove the error bound of the Stein estimator with its original out-of-sample extension.
\begin{proof}[Proof of \Cref{cor:stein-bound}]
    The Stein estimator at point $\bx \in \cX$ can be written as 
    \[ \hat \bs_{p,\lambda, \bx}(\bx) = \sum_{i=1}^d \innerz{\cK_\bx \be_i}{\hat\bs_{p,\lambda,\bx}}_{\cH} \be_i,
    \]
     where $\{ \be_i \}$ is the standard basis of $\R^d$. Note that
     \[
          \sup_{\bx \in \cX} \normz{\hat \bs_{p,\lambda,\bx}(\bx) - \bs_p(\bx)}_2 \leq \sum_{i=1}^d \sup_{\bx \in \cX} |\innerz{\cK_\bx \be_i}{\hat\bs_{p,\lambda,\bx} - \bs_p}_{\cH}| \leq \kappa^2 \sup_{\bx\in\cX}\normz{\hat\bs_{p,\lambda,\bx} - \bs_p}_{\cH}.
        \]
        Then, the bound of Stein estimator immediately follows from Theorem~\ref{thm:mixture-error-bound}.
\end{proof}

\section{Details in Section~\ref{sec:estimators}} \label{appendix:details-in-sec3}

\subsection{A General Version of Nystr\"om KEF}
\label{appendix:general-nkef}
In this section, we briefly review the Nystr\"om version of KEF (NKEF, \citet{sutherland2017efficient}) and give a more general version of it in our framework. 

One of the drawbacks of KEF, as we have mentioned before, is the high computational complexity. 
It requires to solve an $Md\times Md$ linear system, where $M$ is the sample size and $d$ is the dimension.
Note that the solution of KEF in \eqref{eqn:kef-formula} lies in the subspace generated by $\{ \partial_i k(\bx^m, \cdot) : i \in [d], m \in [M] \} \cup \{ \hat \xi \}$.
The Nystr\"om version of KEF consider to minimize the loss \eqref{eq:kef-sm} in a smaller subspace generated by $\{ \partial_i k(\bz^n, \cdot) : i \in [d], n \in [N] \}$, where $N \ll M$ and $\{ \bz^n \}$ is a subset randomly sampled from $\{\bx^m\}$.
\citet{sutherland2017efficient} showed that it suffices to solve an $Nd \times Nd$ linear system, which reduces the computational complexity, while the convergence rate remains the same as that of KEF if $N = \Omega(M^\theta \log M)$, where $\theta \in [1/3, 1/2]$.

In our framework, we can also consider to find our estimator in a smaller subspace. Let $\cH_\bZ$ be the subspace generated by $\{ \bz^n \}_{n \in [N]}$, i.e., $\cH_\bZ := \mathrm{span}\{ \cK_{\bz^n}\bc : n \in [N], \bc\in \R^d \}$.
Consider the minimization problem, which is a modification of \eqref{eqn:tikhonov-minimization}, where the solution is found in $\cH_\bZ$:
\begin{equation}
\hat{\bs}^\bZ_{p,\lambda} = \argmin_{\bs \in \cH_\bZ} \frac{1}{M}\sum_{m=1}^M \|\bs(\bx^m) - \bs_p(\bx^m)\|_2^2  + \frac{\lambda}{2}\|\bs\|^2_{\mathcal{H}_{\cK}}.
\end{equation}
The solution can be written as $\hat s^\bZ_{p,\lambda} = (P_\bZ\hat L_\cK P_\bZ + \lambda I)^{-1} P_\bZ \hat \zeta$, where $\hat \zeta, \hat L_\cK$ are defined as in Sec.~\ref{sec:unifying-framework} and $P_\bZ: \cH_\cK \to \cH_\cK$ is the projection operator onto $\cH_\bZ$, which can be defined as
\[ P_\bZ f := \argmin_{g \in \cH_\cZ} \norm{g - f}_\cH^2 = S^*_\bZ(S_\bZ S^*_\bZ)^{-1}S_\bZ f, \]
where $S_\bZ, S^*_\bZ$ is the sampling operator and its adjoint, respectively.
This motivates us to define the Nystr\"om version of our score estimators for general regularization schemes as follows:
\begin{equation}
    \label{eqn:general-nystrom-score}
    \hat s_{p,\lambda}^{g,\bZ} := -g_\lambda(P_\bZ\hat L_\cK P_\bZ)P_\bZ\hat \zeta. 
\end{equation}
To obtain the matrix form of \eqref{eqn:general-nystrom-score}, we first introduce two operators:
\[ \begin{aligned}
	\cL &:= P_\bZ \hat L_\cK P_\bZ, \\
	\bL &:= \bK_{\bZ\bZ}^{-\frac{1}{2}} \bK_{\bZ\bX}\bK_{\bX\bZ} \bK_{\bZ\bZ}^{-\frac{1}{2}}.
\end{aligned} \]
We want to connect the spectral decompositions of $\bL$ and $\cL$ as in Lemma~\ref{lemma:eigen-connection}. Suppose the spectral decomposition of $\bL$ is $\sum_{i=1}^{Md} \sigma_i \bu_i\bu_i^\trans$,
where $\norm{\bu_i}_{\R^{Md}} = 1$. Consider $v_i := S^*_\bZ \bK_{\bZ\bZ}^{-\frac{1}{2}}\bu_i$, we can verify that
\[\begin{aligned}
		\norm{v_i}_\cH^2 &= ( \bK_{\bZ\bZ}^{-\frac{1}{2}}\bu_i)^{\trans} \bK_{\bZ\bZ} \bK_{\bZ\bZ}^{-\frac{1}{2}}\bu_i  = \bu_i^\trans\bu_i = 1, \\
		\cL v_i &= S^*_\bZ \bK_{\bZ\bZ}^{-1}\bK_{\bZ\bX}\bK_{\bX\bZ} \bK_{\bZ\bZ}^{-\frac{1}{2}}\bu_i 
		= S^*_\bZ\bK_{\bZ\bZ}^{-\frac{1}{2}} \bL \bu_i
		= \sigma_i v_i.
	\end{aligned} \]
Thus, $\cL = \sum_{i=1}^{Md} \sigma_i \inner{v_i}{\cdot}_\cH v_i$ is the spectral decomposition of $\cL$. The estimator can be written as 
\begin{equation} 
    \label{eqn:matrix-nystrom-score}
    \begin{aligned}
	\hat s^{g,\bZ}_{p,\lambda} = -S^*_\bZ \bK_{\bZ\bZ}^{-\frac{1}{2}}\left (\sum_{i=1}^{Md} g_\lambda(\sigma_i) \bu_i \bu_i^\trans \right )\bK_{\bZ\bZ}^{-\frac{1}{2}} \bh
	= -S^*_\bZ \bK_{\bZ\bZ}^{-\frac{1}{2}}g_\lambda(\bL)\bK_{\bZ\bZ}^{-\frac{1}{2}}\bh.
\end{aligned} \end{equation}
The above estimator only involves smaller matrices. However, it requires some expensive matrix manipulations like the matrix square root for general regularization schemes.
Fortunately, these expensive terms can be cancelled when using the Tikhonov regularization:
\begin{example}
    When we consider the Tikhonov regularization $g_\lambda(\sigma) = (\sigma + \lambda)^{-1}$ and curl-free kernels, the score estimator \eqref{eqn:matrix-nystrom-score} becomes
    $\hat s^{g,\bZ}_{p,\lambda}(\bx)
    = -\bK_{\bx\bZ} \bK_{\bZ\bZ}^{-\frac{1}{2}}(\bL + \lambda I)^{-1}\bK_{\bZ\bZ}^{-\frac{1}{2}}\bh 
    = -\bK_{\bx\bZ} ( \bK_{\bZ\bX}\bK_{\bX\bZ} + \lambda \bK_{\bZ\bZ})^{-1} \bh$.
    Similar to Example~\ref{example:kef}, we find this is exactly the same as the NKEF estimator obtained in \citet[Theorem 1]{sutherland2017efficient}.
\end{example}

\subsection{Computational Details} \label{appendix:computational-details}
    
\paragraph{Details of Example~\ref{example:ssge}} 
Using the notation in Example~\ref{example:ssge} and Sec.~\ref{sec:kexpf}, we can reformulate SSGE into a matrix form as follows:
\[ \begin{aligned}
    \hat g_i(\bx) &= - \sum_{j=1}^J \left ( \frac{1}{M} \sum_{n = 1}^M \partial_i \hat \psi_j(\bx^n) \right ) \psi_j(\bx) \\
    &= - \sum_{j=1}^J \frac{1}{M} \left ( \frac{\sqrt M}{\lambda_j} \sum_{n, m=1}^M \partial_i k(\bx^n, \bx^m) w_j^{(m)} \right ) \left ( \frac{\sqrt M}{\lambda_j} \sum_{\ell=1}^M  k(\bx, \bx^\ell) w_j^{(\ell)} \right ) \\
    &= - \sum_{j=1}^J \frac{1}{\lambda_j^2} \left ( \sum_{n, m=1}^M \partial_i k(\bx^n, \bx^m) w_j^{(m)} \right ) \left (  \sum_{\ell=1}^M  k(\bx, \bx^\ell) w_j^{(\ell)} \right ) \\
    &= - \sum_{\ell=1}^M k(\bx, \bx^\ell) \sum_{n,m=1}^M \left ( \sum_{j=1}^J \frac{w_j^{(m)}w_j^{(\ell)}}{\lambda_j^2} \right )   \partial_i k(\bx^n, \bx^m)  \\
    &= - \sum_{\ell=1}^M k(\bx, \bx^\ell) \sum_{m=1}^M \left ( \sum_{j=1}^J \frac{w_j^{(m)}w_j^{(\ell)}}{\lambda_j^2}  \right )   \left ( \sum_{n=1}^M\partial_i k(\bx^n, \bx^m)\right )  \\
    &= -k(\bx,\bX) \left ( \sum_{j=1}^J \frac{\bw_j\bw_j^\trans}{\lambda_j^2} \right ) \br_i,
\end{aligned} \]
where $r_{i,j} = \sum_{n=1}^M \partial_i k(\bx^n, \bx^j)$, and $\bw_1, \cdots, \bw_M$ is the unit eigenvectors of $k(\bX, \bX)$ corresponding to eigenvalues $\lambda_1 \geq \cdots \geq \lambda_M$. $w_j^{(m)}$ is the $m$-th component of $\bw_j$. Note that when using diagonal kernels, we have $\cK(\bx, \by) = k(\bx, \by) \otimes \mathbf I_d$, then the eigenvectors of $\cK(\bX, \bX)$ are $\{ \bw_i \otimes \be_j : i \in [M], j \in [d] \}$ and the eigenvalue corresponds to $\bw_i \otimes \be_j$ is $\lambda_i$, where $\{ \be_j \}$ is the standard basis of $\R^d$. We also note that in this case
\[ h_{(m-1)d+i} = \hat\bzeta(\bx^m)_i = \frac{1}{M}\sum_{\ell=1}^M (\divgers{\bx^\ell} \cK(\bx^\ell, \bx^m))_i = \frac{1}{M}\sum_{\ell=1}^M \partial_i k(\bx^\ell, \bx^m) = M r_{i,m}. \]
Comparing with \refeq{eqn:spectral-cutoff-regularized-estimator}, we find that SSGE is equivalent to use diagonal kernels and spectral cut-off regularization.

\paragraph{Details of Example~\ref{example:stein}} 
For the regularizer $g_\lambda(\sigma) := (\lambda + \sigma)^{-1} \bone_{\{\sigma>0\}}$, from \Cref{lemma:general-non-zero-regularizer} we know when $\bK$ is non-singular, $\hat\bs_{p,\lambda}^g(\bx) = -\bK_{\bx\bX}\bK^{-1}(\frac{1}{M} \bK + \lambda\mathbf I)^{-1} \bh$. Next, we consider the minimization problem in \refeq{eqn:tikhonov-minimization}, and ignore the one-dimensional subspace $\R\hat\bzeta$ of the solution space, and assume the solution is $\bK_{\bx\bX}\bc$ as before. We can rewrite the objective in \refeq{eqn:tikhonov-minimization} to
\[ \frac{1}{M}\bc^\trans\bK^2\bc + \lambda\bc\bK\bc + 2\bc^\trans\bh. \]
By taking gradient, we find $\bc$ satisfies $(\frac{1}{M}\bK^2+\lambda\bK)\bc = -\bh$, so it is equivalent to use the previously mentioned regularization.

\subsection{Curl-Free Kernels}  \label{appendix:curl-free}

\paragraph{Recover the Function From Its Gradient.}
Since vector fields in a curl-free RKHS is always the gradient of some functions, it is possible to recover these functions from its gradient.
Specifically, suppose the curl-free kernel is defined by $\cKcf(\bx, \by) = -\nabla^2 \psi(\bx - \by)$ and $f \in \cH_{\cKcf}$. 
Assume $f$ is of the following form
\[ f = \sum_{i=1}^m \cKcf(\bx^i, \cdot)\bc_i 
= -\sum_{i=1}^m \sum_{j=1}^d \nabla(\partial_j\psi(\bx^i - \cdot))c_i^{(j)} 
= \nabla \left ( -\sum_{i=1}^m \sum_{j=1}^d \partial_j\psi(\bx^i - \cdot)c_i^{(j)}  \right )
, \]
where $c_i^{(j)}$ is the $j$-th component of $\bc_i$.
Then, we find a desired function whose gradient is $f$.

\paragraph{The Special Structure of $\cKcf(\bx, \by) = -\nabla^2 \phi(\norm{\bx-\by})$.}
As we have mentioned in Sec.~\ref{sec:scalability},  curl-free kernels have some special structures.
Suppose $\cKcf$ is a curl-free kernel defined by $\nabla^2 \phi(r)$, where $\br = (\bx - \bx^\prime)^T$ and $r = \norm{\br}$. Then
	\[ \begin{aligned}
			\frac{\partial}{\partial r_i} \phi & = \phi^\prime \frac{r_i}{r}, \\
			\nabla \frac{\partial}{\partial r_i} \phi & = \phi^{\prime\prime}  \frac{r_i}{r^2} \br + \phi^\prime \frac{\be_i r - r_i\frac{\br}{r}}{r^2}, 
	\end{aligned} \]
    where $\be_i$ is the $i$-th column of the identity matrix. Then the curl-free kernel is of the form
    \begin{equation}
        \label{eqn:curl-free-rbf}
        \cKcf(\bx, \by) = \left( \frac{\phi^\prime}{r^3} - \frac{\phi^{\prime\prime}}{r^2} \right) \br\br^\trans - \frac{\phi^\prime}{r} \mathbf I.
    \end{equation}
    
    We also obtain a divergence formula for such kernel. Note that 
	\[ \begin{aligned}
		\partial_{jj} \partial_i \phi &= \phi^{\prime\prime\prime} \frac{r_j^2r_i}{r^3} + \phi^{\prime\prime}\frac{(r_i + r_j\delta_{ij})r^2 - 2r_j^2r_i}{r^4} \\
		& + \phi^{\prime\prime} \frac{r_j}{r}\frac{\delta_{ij}r - r_i\frac{r_j}{r}}{r^2} 
		 + \frac{\phi^\prime}{r^6}\left[ (\delta_{ij}r_j - r_i) r^3 - 3rr_j(\delta_{ij}r^2 - r_ir_j) \right],
\end{aligned} \]
	where $\delta_{ij} = [i = j]$. Next, we sum out $j$ and then obtain
    \begin{equation}
        \label{eqn:curl-free-div}
		\divger_\bx \cKcf(\bx, \bx^\prime) = -\Delta(\partial_i \phi)(r) =
		 -\frac{\br}{r} \left [
			\phi^{\prime\prime\prime}(r)
			+ \frac{d-1}{r} \left ( \phi^{\prime\prime}(r) - \frac{\phi^\prime(r)}{r} \right )
		\right ].
    \end{equation}

\paragraph{The Special Structure of $\cKcf(\bx, \by) = -\nabla^2 \varphi(\norm{\bx-\by}^2)$.}
Since many frequently used kernels only depend on $\norm{\bx - \by}^2$, we consider the structure of curl-free kernels of these types.
Suppose $\cKcf$ is a curl-free kernel defined by $\nabla^2 \varphi(r^2)$, where $\br = (\bx - \bx^\prime)^T$ and $r = \norm{\br}$. 
Then, using \eqref{eqn:curl-free-rbf} and \eqref{eqn:curl-free-div} we can find
 \begin{align}
    \cKcf(\bx, \by) &= -4\varphi^{\prime\prime} \br\br^\trans - 2\varphi^\prime\mathbf{I}, \\
    \divgers{\bx} \cKcf(\bx, \by) &= -4[(d+2)\varphi^{\prime\prime} + 2r^2\varphi^{\prime\prime\prime}]\br.
\end{align} 

\subsection{Details of Different Regularization Schemes} \label{sec:proofs}

\subsubsection{Tikhonov Regularization}
\begin{proof}[Proof of Theorem~\ref{thm:tik-score}]
    When $g_\lambda(\sigma) = (\sigma + \lambda)^{-1}$, the estimator is $\hat \bs_{p,\lambda} = -(\hat L_\cK + \lambda I)^{-1}\hat\bzeta$. We need to compute the explicit formula of the inverse of $\hat L_\cK + \lambda I$. Note that $(\hat L_\cK + \lambda I)^{-1} \hat \bzeta$ is the solution of the following minimization problem
    \[
        \hat \bs_{p,\lambda}^g= \argmin_{\bs \in \cH_\cK} \frac{1}{M} \sum_{i=1}^M \bs(\bx^i)^\trans\bs(\bx^i) + 2 \innerz{\bs}{\hat \bzeta}_\cH  + \lambda \norm{\bs}_\cH^2. 
        \]

    From the general representer theorem~\citep[Theorem A.2]{sriperumbudur2017density}, the minimizer lies in the space generated by
     \[ \{ \cK_{\bx^i}\bc : i \in [M], \bc \in \R^d \} \cup \{ \hat \bzeta \}. \]

    We can assume
    \[ \hat \bs_{p,\lambda}^g = \sum_{i=1}^M \cK_{\bx^i} \bc_i + a\hat \bzeta. \]

    Define $\bc := (\bc_1, \cdots, \bc_M)$ and $\bh := (\hat \bzeta(\bx^1), \cdots, \hat \bzeta(\bx^M))$, then the optimization objective can be written as 
    \[ \frac{1}{M}(\bc^\trans \bK^2 \bc + 2a\bc^\trans \bK \bh + a^2 \bh^\trans\bh) + 2(a\normz{\hat \bzeta}_\cH^2 + \bh^\trans\bc) + \lambda(\bc^\trans \bK\bc + 2a\bc^\trans\bh + a^2\normz{\hat \bzeta}_\cH^2). \]

    Taking the derivative, we need to solve the following linear system
    \[\begin{aligned}
        \frac{1}{M}(\bK^2\bc + a\bK\bh) + \bh + \lambda(\bK\bc+a\bh) &= 0, \\
        \frac{1}{M}(a\bh^\trans\bh+\bc^\trans\bK\bh) + (1 + \lambda a)\normz{\hat\bzeta}_\cH^2 + \lambda\bc^\trans\bh &= 0.
    \end{aligned} \]

    By some calculations, this system is equivalent to $a = -1/\lambda$ and $(\bK+M\lambda I)\bc = \bh/\lambda$.
\end{proof}

\subsubsection{Spectral Cut-Off Regularization}
\begin{proof}[Proof of Lemma~\ref{lemma:eigen-connection}]
	Let $\cH_0$ be the subspace of $\cH_\cK$ generated by $\{ \cK_{\bx^m}\bc$ : $\bc \in \R^d$, $m \in [M] \}$. Note that $f(\bx^m)^\trans \bc = \inner{\cK(\cdot, \bx^m)\bc}{f}_\cH = 0$ for any $f \in \cH_0^\perp$ and $\bc\in \R^d$. We know $\hat L_\cK = 0$ on $\cH_0^\perp$. Also note $\hat L_\cK v \in \cH_0$ and $v(\bx^m) = \bu^{(m)}\sqrt{M\sigma}$, then
	\[ \hat L_\cK v(\bx^k) = \frac{1}{M}\sum_{m=1}^M \cK(\bx^k, \bx^m)v(\bx^m) = \frac{1}{\sqrt{M}} \sum_{m=1}^M \cK(\bx^k, \bx^m) \sqrt{\sigma}\bu^{(m)} = \sigma v(\bx^k),  \]
	and we conclude that $\hat L_\cK v = \sigma v$. The following equation shows $v$ is normalized:
	\[ \begin{aligned}
			\norm{ v}^2_\cH &= \frac{1}{\sqrt{M\sigma}}\sum_{m=1}^M \inner{ \cK(\cdot, \bx^m) \bu^{(m)}}{ v}_\cH  = \frac{1}{\sqrt{M\sigma}}\sum_{m=1}^M \inner{ \bu^{(m)}}{ v(\bx^m)}_{\R^d} = \sum_{m=1}^M ( \bu^{(m)})^\trans  \bu^{(m)} = 1.
\end{aligned}\] 
\end{proof}

Theorem~\ref{thm:spectral-score} is a corollary of the following lemma, which provides a general form for the regularizer $g_\lambda$ with $g_\lambda(0) = 0$. 
\begin{lemma}
    \label{lemma:general-non-zero-regularizer}
    Let $g_\lambda: [0, \kappa^2] \to \R$ be a regularizer such that $g_\lambda(0) = 0$.
    Let $(\sigma_j, \bu_j)_{j\geq 1}$ be the non-zero eigenvalue and eigenvector pairs that satisfy $\frac{1}{M}\bK\bu_j = \sigma_j\bu_j$.
    Then we have
    \[ g_\lambda(\hat L_\cK) \hat \bzeta = \bK_{\bx\bX} \left ( \sum \frac{g_\lambda(\sigma_i)}{M\sigma_i} \bu_i \bu_i^\trans \right ) \bh, \]
    where $\bK_{\bx\bX}$ and $\bh$ are defined as in \Cref{thm:tik-score}.
\end{lemma}
\begin{proof}
    Let $\{ (\mu_i, v_i ) \}$ be the pairs of non-zero eigenvalues and eigenfunctions of $\hat L_\cK: \cH \to \cH$, then by Lemma~\ref{lemma:eigen-connection} we have $\sigma_i = \mu_i$. Note that
    \[ \hat L_\cK = \sum \mu_i \innerz{v_i}{\cdot}_\cH v_i \quad \text{and} \quad g_\lambda(\hat L_\cK) = \sum g_\lambda(\mu_i) \innerz{v_i}{\cdot}_\cH v_i. \]

    From Lemma \ref{lemma:eigen-connection}, we have

    \[ \begin{aligned}
         g_\lambda(\hat L_\cK)\hat\bzeta
      &= \sum g_\lambda(\sigma_i) \innerz{v_i}{\hat\bzeta}_\cH v_i \\
      &= \sum \left\{ g_\lambda(\sigma_i) \inner{\frac{1}{\sqrt{M\sigma_i}} \sum_{j=1}^M \cK_{\bx^j}\bu^{(j)}_i}{\hat\bzeta}_\cH\frac{1}{\sqrt{M\sigma_i}} \sum_{k=1}^M \cK_{\bx^k}\bu^{(k)}_i \right \} \\
      &= \frac{1}{M} \sum \sum_{j,k=1}^M g_\lambda(\sigma_i)\sigma_i^{-1} \inner{\cK_{\bx^j}\bu^{(j)}_i}{\hat\bzeta}_\cH \cK_{\bx^k}\bu^{(k)}_i \\
      &= \frac{1}{M} \sum \sum_{j,k=1}^M g_\lambda(\sigma_i)\sigma_i^{-1} \hat\bzeta(\bx^j)^\trans \bu^{(j)}_i  \cK_{\bx^k}\bu^{(k)}_i \\
      &=  \cK_{\bx \bX} \left ( \sum \frac{g_\lambda(\sigma_i)}{M\sigma_i}\bu_i\bu_i^\trans \right )\bh.
    \end{aligned} \]
\end{proof}

\subsubsection{Iterative Regularization} \label{appendix:iterative-regularization}
\begin{theorem}[Landweber iteration]
    Let $\hat{\bs}_{p,\lambda}^g$ be defined as in \eqref{eqn:general-score-estimator}, and $g_\lambda(\sigma) = \eta \sum_{i=0}^{t-1} (1 - \eta \sigma)^i$, where $t := \lfloor \lambda^{-1} \rfloor$.
    Then we have \[\hat \bs_{p,\lambda}^g(\bx) = -t\eta \hat \bzeta(\bx) + \bK_{\bx\bX}\bc_t,\] where $\bc_0 = 0$ and $\bc_{t+1} = (\mathbf{I}_d - \eta \bK / M)\bc_t - t\eta^2 \bh / M$, and $\bK_{\bx\bX}$ and $\bh$ are defined as in \Cref{thm:tik-score}.
\end{theorem}
\begin{proof}
We note that the iteration process is
\[\begin{aligned}
     \hat \bs_{p}^{(1)} &= -\eta \hat\bzeta, \\
     \hat \bs_{p}^{(t)} &= -\eta \hat\bzeta + (I - \eta\hat L_\cK)\hat \bs_{p}^{(t-1)} \\
      &= \hat \bs_{p}^{(t-1)} + \eta ( -\hat \bzeta - \hat L_\cK \hat \bs_{p}^{(t-1)}).
\end{aligned} \]

where we define $\hat \bs_p^{(t)} := \hat \bs_{p,1/t}$. We can assume
\[ \hat \bs_{p}^{(t)} = a_t \hat\bzeta + \bK_{\bx\bX}\bc_t. \]

Then, by induction, 
\[ \begin{aligned} 
    \hat \bs_{p}^{(t)} &= -\eta \hat\bzeta + (I - \eta\hat L_\cK)(a_{t-1} \hat\bzeta + \bK_{\bx\bX}\bc_{t-1}) \\
    &= (a_{t-1} - \eta)\hat\bzeta + \bK_{\bx\bX} ( \bc_{t-1} + \eta a_{t-1} \bh / M  - \eta \bK \bc_{t-1} / M).
\end{aligned} \]

Thus, we have $a_t = -t\eta$ and $\bc_t =  (\mathbf I - \eta \bK / M) \bc_{t-1} - (t-1)\eta^2 \bh / M$, and $\bc_1 = 0$.
\end{proof}

Before introducing the $\nu$-method, we recall that the iterative regularization can be represented by a family of polynomials $g_\lambda(\sigma) = \mathrm{poly}(\sigma)$, where $g_\lambda$ converges to the function $1/\sigma$ as $\lambda \to 0$. 
For example, in the Landweber iteration we see that 
\[ g_\lambda(\sigma) = \eta\sum_{i=0}^{t - 1}(1 - \eta\sigma)^i = \frac{1 - (1 - \eta\sigma)^t}{\sigma}. \]
We can verify that the identification of $\lambda$ and $t^{-1}$ satisfies Definition~\ref{def:regularizer} about the regularization.
To see the qualification, we note that the maximum $|1 - \sigma g_\lambda(\sigma)|\sigma^{r} = \sigma^{r} (1 - \eta\sigma)^t$ over $[0, \eta^{-1}]$ is attained when $\sigma = r / (r \eta + t)$ and hence
\[ \sup_{0 \leq \sigma \leq \eta^{-1}} |1 - \sigma g_\lambda(\sigma)|\sigma^{r} \leq \frac{t^tr^r}{(r\eta+t)^{r + t}} \leq \left ( \frac{r}{t} \right )^{r} = \max(r^r, 1) \lambda^r.  \]
Thus, we see that the qualification is $\infty$. 

\begin{example}[$\nu$-method]
    \label{example:nu-method}
    The $\nu$-method~\citep{engl1996regularization} is an accelerated version of the Landweber iteration.
    The idea behind it is to find better polynomials $p_t(\sigma)$ to approximate the function $1/\sigma$, where $p_t$ is a polynomial of degree $t$.
    These polynomials satisfy $\sup_{0 \leq \sigma \leq 1} |1 - \sigma p_t(\sigma)|\sigma^\nu \leq c_\nu t^{2\nu}$. 
    Compared with the definition of the qualification in Definition~\ref{def:regularizer}, we can identify $\lambda$ and $t^{-2}$. 
    Thus, for the same regularization parameter, the $\nu$-method only requires about $\lambda^{-1/2}$ iterations while the Landweber iteration requires about $\lambda^{-1}$ iterations.
    For more details about the construction of these polynomials, we refer the readers to \citet[Appendix A.1 and Section 6.3]{engl1996regularization}
    
    Below we give the algorithm of the $\nu$-method, where $t = \lfloor \lambda^{-1/2} \rfloor$ and $\hat \bs_{p, \lambda} := \hat \bs_p^{(t)}$. 

\[\begin{aligned}
     \hat \bs_{p}^{(0)} &= 0,  \quad 
     \hat \bs_{p}^{(1)} = -\omega_1 \hat\bzeta,  \\
     \hat \bs_{p}^{(t)} &= \hat \bs_{p}^{(t-1)}
        + u_t(\hat \bs_p^{(t-1)} - \hat \bs_p^{(t-2)})
        + \omega_t(-\hat\bzeta - \hat L_\cK \hat \bs_p^{(t-1)}),
\end{aligned} \]
where 
\[ \begin{aligned}
    u_t &= \frac{(t-1)(2t-3)(2t+2\nu-1)}{(t+2\nu-1)(2t+4\nu-1)(2t+2\nu-3)},\\
    \omega_t &= \frac{4(2t+2\nu-1)(t+\nu-1)}{(t+2\nu-1)(2t+4\nu-1)}.
\end{aligned} \]

Smilarly, we can assume
\[ \hat \bs_{p}^{(t)} = a_t \hat\bzeta + \bK_{\bx\bX}\bc_t. \]

Then, by induction, 
\[ \begin{aligned} 
     \hat \bs_p^{(t)} &= \left (1+u_t - \omega_t \hat L_\cK\right) \hat \bs_p^{(t-1)}
        - u_t \hat \bs_p^{(t-2)} 
        - \omega_t \hat\bzeta \\
      &= \left (1+u_t - \omega_t \hat L_\cK\right) (a_{t-1} \hat\bzeta + \bK_{\bx\bX}\bc_{t-1})
        - u_t (a_{t-2} \hat\bzeta + \bK_{\bx\bX}\bc_{t-2})
        - \omega_t \hat\bzeta \\
    &= \left ( (1+u_t) a_{t-1} - u_ta_{t-2} - \omega_t  \right)\hat\bzeta \\
    &\quad + \bK_{\bx\bX} \left ( 
        (1+u_t)\bc_{t-1}
        - \frac{\omega_t}{M}(a_{t-1}\bh + \bK\bc_{t-1})
        - u_t \bc_{t-2}
      \right ).
\end{aligned} \]

Thus, we obtain the iteration formula for $a_t$ and $\bc_t$ as follows:
\[ \begin{aligned}
    a_t &:= (1 + u_t)a_{t-1} - u_ta_{t-2} - \omega_t, \\
       \bc_t &:=(1+u_t)\bc_{t-1}
        - \frac{\omega_t}{M}(a_{t-1}\bh + \bK\bc_{t-1})
        - u_t \bc_{t-2},
\end{aligned} \]
and $\bc_0 = \bc_1 = 0$, $a_0 = 0$, $a_1 = -\omega_1$.
\end{example}

\section{Technical Results} \label{appendix:technique-results}

\begin{lemma}
    \label{lemma:integral-operator-trace-norm}
    Suppose Assumption~\ref{assumption:bounded-trace} holds, then $L_\cK, \hat L_\cK: \cH_\cK \to \cH_\cK$ are positive, self-adjoint, trace class operators. Moreover, $\tr L_\cK \leq \kappa^2$ and $\tr \hat L_\cK \leq \kappa^2$.
\end{lemma}
\begin{proof}
    The result follows from a simple calculation. It is easy to see $L_\cK$ and $\hat L_\cK$ are positive and self-adjoint. We prove they are in trace class. Let $\{ \varphi_i\}$ be a orthonormal basis of $\cH_\cK$ and $\{ \be_i \}$ be the standard basis of $\R^d$, then
    \[ \begin{aligned}
         \tr L_\cK &= \sum_i \innerz{L_\cK \varphi_i}{\varphi_i}_\cH 
         = \int_\cX \sum_i \innerz{ \cK_\bx \varphi_i }{ \varphi_i }_\cH d\rho 
          = \sum_{k=1}^d \int_\cX \sum_i \innerz{ \innerz{\cK_\bx \be_k}{\varphi_i}_\cH \cK_\bx \be_k }{ \varphi_i }_\cH d\rho     \\
         & = \sum_{k=1}^d \int_\cX \sum_i |\innerz{\cK_\bx \be_k}{\varphi_i}_\cH|^2 d\rho     
          = \sum_{k=1}^d \int_\cX \normz{\cK_\bx \be_k}_\cH^2 d\rho     
          = \int_\cX \tr \cK(\bx, \bx) d\rho \leq \kappa^2
    \end{aligned} \]
    Similarly, we have $\tr \hat L_\cK \leq \kappa^2$.
\end{proof}

We need the following concentration inequality in Hilbert spaces used in~\citet{bauer2007regularization}.

\begin{lemma}[\citet{bauer2007regularization}, Proposition 23]
    \label{lemma:hilbert-space-concentration}
    Let $\xi$ be a random variable with values in a real Hilbert space $H$. Assume there are two constants $\sigma, H$, such that
    \[ \E [ \normz{\xi - \E\xi}_H^m ] \leq \frac{1}{2}m! \sigma^2 H^{m-2}, \quad \forall m \geq 2.\]
    Then, for all $n \in \N$, $0 < \delta < 1$, the following inequality holds with probability at least $1 - \delta$
    \[ \normz{\hat \xi - \E\xi}_H \leq 2 \left ( \frac{H}{n} + \frac{\sigma}{\sqrt n} \right ) \log \frac{2}{\delta}, \]
    where $\hat\xi = \frac{1}{n}\sum_{i=1}^n \xi_i$ and $\{ \xi_i \}$ are independent copies of $\xi$.
\end{lemma}

\begin{lemma}
    \label{lemma:xi-concentration}
    Under Assumption~\ref{assumption:bounded-div}, we have for all $M\in \N, 0 < \delta < 1$, the following inequality holds with probability at least $1 - \delta$
    \begin{equation}
        \label{eqn:xi-concentration}
        \normz{\hat \bzeta - \bzeta}_\cH \leq 2\left ( \frac{K}{M} + \frac{\Sigma}{\sqrt M} \right ) \log \frac{2}{\delta}, \\
    \end{equation}
    where $\hat \zeta = \frac{1}{M}\sum_{m=1}^M \divgers{\bx^m}\cK^\trans_{\bx^m}$ and $\{ \bx^m\}$ is the set of i.i.d. samples from $\rho$.
\end{lemma}

\begin{proof}
    Define an $\cH_\cK$-valued random variable $\xi_\bx := \divgers{\bx} \cK^\trans_\bx$. It is easy to see $\E_{\bx\sim\nu}[ \xi_\bx ]= -L_\cK \bs_p =: \xi$. From Assumption~\ref{assumption:bounded-div}, we have for $m \geq 2$, 
    \[ \E_{\nu} [ \normz{ \xi_\bx - \xi }_\cH^m]
     \leq m! K^m \E_\nu \left [ \exp\left (\frac{\normz{ \xi_\bx - \xi }_\cH}{K}\right)  - \frac{\normz{ \xi_\bx - \xi }_\cH}{K} - 1\right ] \leq \frac{1}{2} m! \Sigma^2 K^{m-2}. \]
     Note that $\hat \bzeta = \frac{1}{M}\sum_{m=1}\xi_{\bx^m}$ and $\E_\nu \hat \bzeta = \xi$. Then \refeq{eqn:xi-concentration} follows from Lemma~\ref{lemma:hilbert-space-concentration}.
\end{proof}

\begin{lemma}
    \label{lemma:L-concentration}
    Under Assumption~\ref{assumption:bounded-trace}, we have for all $M\in \N, 0 < \delta < 1$, the following inequality holds with probability at least $1 - \delta$
    \begin{equation}
        \normz{\hat L_\cK - L_\cK}_\cH \leq \frac{2\sqrt 2 \kappa^2}{\sqrt M} \sqrt{\log \frac{2}{\delta}}.
    \end{equation}
\end{lemma}
\begin{proof}
    This is a direct consequence of \citet[Lemma 8]{vito2005learning} and Lemma~\ref{lemma:integral-operator-trace-norm}.
\end{proof}

The following useful lemma is from~\citet[Lemma 7]{de2014learning} and \citet[Lemma 15]{sriperumbudur2017density}
\begin{lemma}
    \label{lemma:operator-lipechitz}
    Suppose $S$ and $T$ are two self-adjoint Hilbert-Schmidt operators on a separable Hilbert space $H$ with spectrum contained in the interval $[a, b]$. Given a Lipschitz function $r: [a, b] \to \R$ with Lipschitz constant $L_r$, we have
    \[ \hsnorm{r(S) - r(T)} \leq L_r \hsnorm{S - T}. \]
\end{lemma}

 \def\generatesamples{0}
\ifx\generatesamples\undefined
\else
\clearpage
\section{Samples} \label{appendix:samples}
\begin{table}[H]
	\centering
    \caption{WAE samples on MNIST.}
    \label{table:samples/mnist}
    \vskip 0.05in
	\begin{tabular}{ccccc}
		 & $d = 8$ & $d = 32$ & $d = 64$ & $d = 128$ \\
		\rotatebox{90}{Stein} &
		\parbox[c]{0.18\textwidth}{\includegraphics[width=\linewidth]{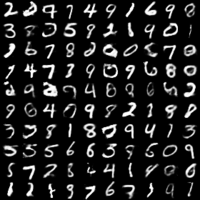}}  &
		\parbox[c]{0.18\textwidth}{\includegraphics[width=\linewidth]{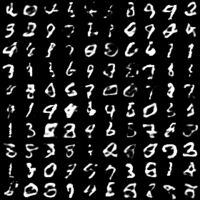}}  &
		\parbox[c]{0.18\textwidth}{\includegraphics[width=\linewidth]{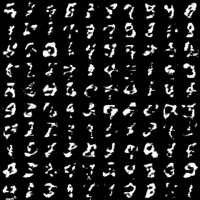}} &
		\parbox[c]{0.18\textwidth}{\includegraphics[width=\linewidth]{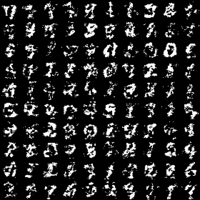}} \\ \\
		\rotatebox{90}{SSGE} &
		\parbox[c]{0.18\textwidth}{\includegraphics[width=\linewidth]{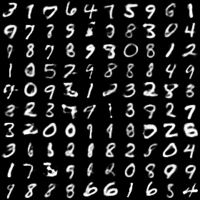}}  &
		\parbox[c]{0.18\textwidth}{\includegraphics[width=\linewidth]{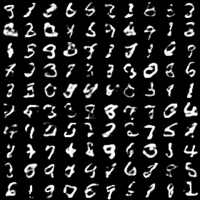}}  &
		\parbox[c]{0.18\textwidth}{\includegraphics[width=\linewidth]{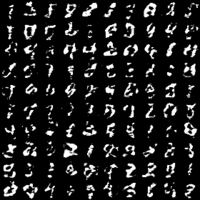}} &
		\parbox[c]{0.18\textwidth}{\includegraphics[width=\linewidth]{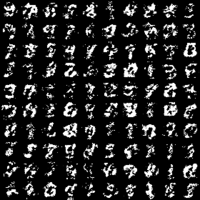}} \\ \\
		\rotatebox{90}{SSM} &
		\parbox[c]{0.18\textwidth}{\includegraphics[width=\linewidth]{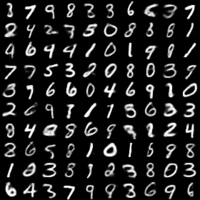}}  &
		\parbox[c]{0.18\textwidth}{\includegraphics[width=\linewidth]{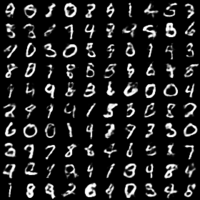}}  &
		\parbox[c]{0.18\textwidth}{\includegraphics[width=\linewidth]{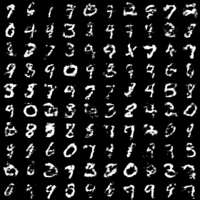}} &
		\parbox[c]{0.18\textwidth}{\includegraphics[width=\linewidth]{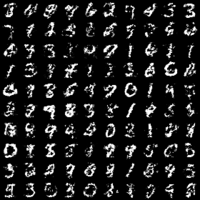}} \\ \\
		\rotatebox{90}{NKEF$_2$} &
		\parbox[c]{0.18\textwidth}{\includegraphics[width=\linewidth]{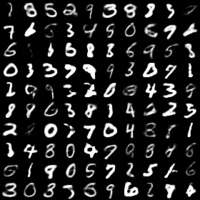}}  &
		\parbox[c]{0.18\textwidth}{\includegraphics[width=\linewidth]{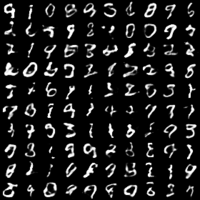}}  &
		\parbox[c]{0.18\textwidth}{\includegraphics[width=\linewidth]{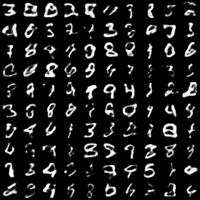}} &
		\parbox[c]{0.18\textwidth}{\includegraphics[width=\linewidth]{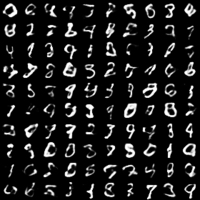}} \\ \\
		\rotatebox{90}{$\nu$-method} &
		\parbox[c]{0.18\textwidth}{\includegraphics[width=\linewidth]{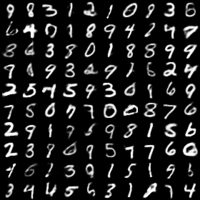}}  &
		\parbox[c]{0.18\textwidth}{\includegraphics[width=\linewidth]{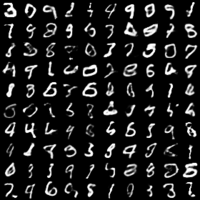}}  &
		\parbox[c]{0.18\textwidth}{\includegraphics[width=\linewidth]{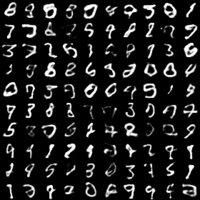}} &
		\parbox[c]{0.18\textwidth}{\includegraphics[width=\linewidth]{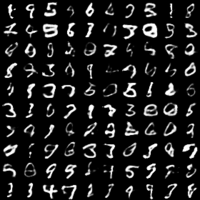}} \\ \\
		\rotatebox{90}{KEF-CG} &
		\parbox[c]{0.18\textwidth}{\includegraphics[width=\linewidth]{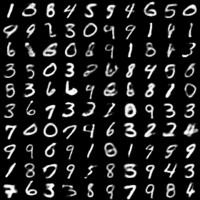}}  &
		\parbox[c]{0.18\textwidth}{\includegraphics[width=\linewidth]{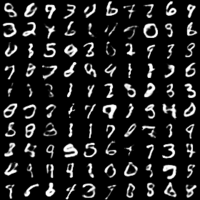}}  &
		\parbox[c]{0.18\textwidth}{\includegraphics[width=\linewidth]{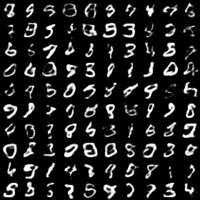}} &
		\parbox[c]{0.18\textwidth}{\includegraphics[width=\linewidth]{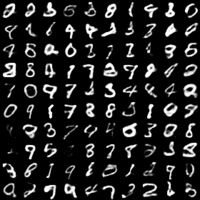}} \\
	\end{tabular}
\end{table}
\begin{table}[H]
	\centering
    \caption{WAE samples on CelebA.}
    \label{table:samples/celeba}
    \vskip 0.05in
	\begin{tabular}{ccccc}
		 & $d = 8$ & $d = 32$ & $d = 64$ & $d = 128$ \\
		\rotatebox{90}{Stein} &
		\parbox[c]{0.18\textwidth}{\includegraphics[width=\linewidth]{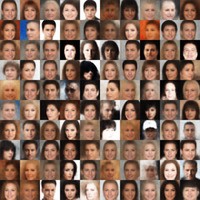}}  &
		\parbox[c]{0.18\textwidth}{\includegraphics[width=\linewidth]{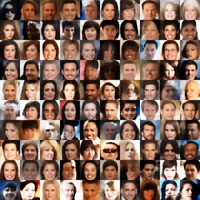}}  &
		\parbox[c]{0.18\textwidth}{\includegraphics[width=\linewidth]{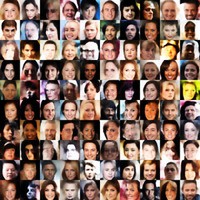}} &
		\parbox[c]{0.18\textwidth}{\includegraphics[width=\linewidth]{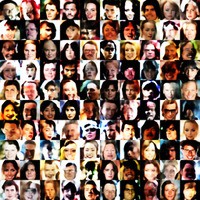}} \\ \\
		\rotatebox{90}{SSGE} &
		\parbox[c]{0.18\textwidth}{\includegraphics[width=\linewidth]{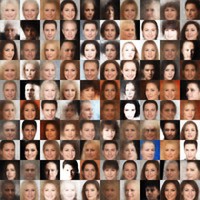}}  &
		\parbox[c]{0.18\textwidth}{\includegraphics[width=\linewidth]{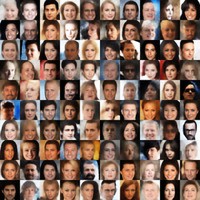}}  &
		\parbox[c]{0.18\textwidth}{\includegraphics[width=\linewidth]{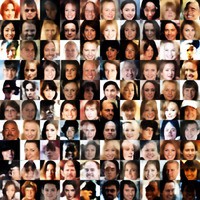}} &
		\parbox[c]{0.18\textwidth}{\includegraphics[width=\linewidth]{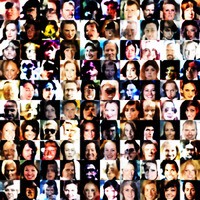}} \\ \\
		\rotatebox{90}{SSM} &
		\parbox[c]{0.18\textwidth}{\includegraphics[width=\linewidth]{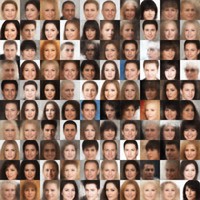}}  &
		\parbox[c]{0.18\textwidth}{\includegraphics[width=\linewidth]{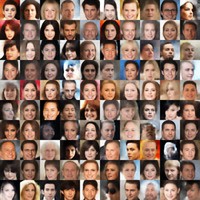}}  &
		\parbox[c]{0.18\textwidth}{\includegraphics[width=\linewidth]{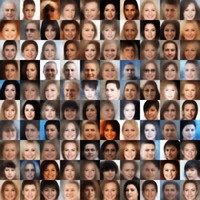}} &
		\parbox[c]{0.18\textwidth}{\includegraphics[width=\linewidth]{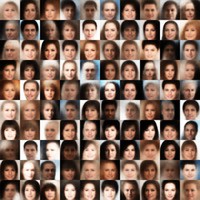}} \\ \\
		\rotatebox{90}{NKEF$_2$} &
		\parbox[c]{0.18\textwidth}{\includegraphics[width=\linewidth]{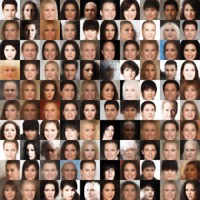}}  &
		\parbox[c]{0.18\textwidth}{\includegraphics[width=\linewidth]{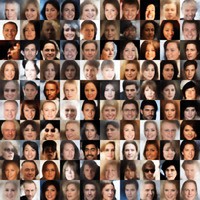}}  &
		\parbox[c]{0.18\textwidth}{\includegraphics[width=\linewidth]{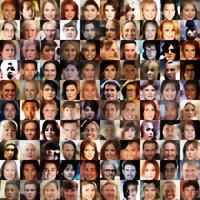}} &
		\parbox[c]{0.18\textwidth}{\includegraphics[width=\linewidth]{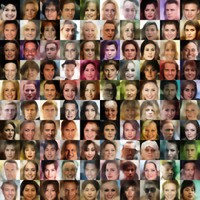}} \\ \\
		\rotatebox{90}{$\nu$-method} &
		\parbox[c]{0.18\textwidth}{\includegraphics[width=\linewidth]{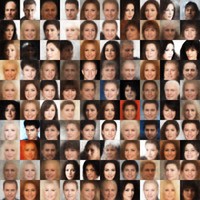}}  &
		\parbox[c]{0.18\textwidth}{\includegraphics[width=\linewidth]{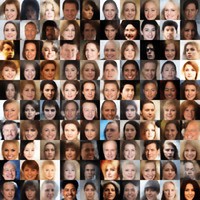}}  &
		\parbox[c]{0.18\textwidth}{\includegraphics[width=\linewidth]{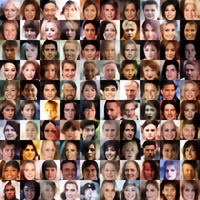}} &
		\parbox[c]{0.18\textwidth}{\includegraphics[width=\linewidth]{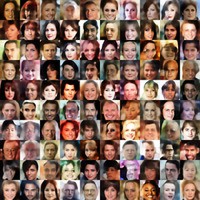}} \\ \\
		\rotatebox{90}{KEF-CG} &
		\parbox[c]{0.18\textwidth}{\includegraphics[width=\linewidth]{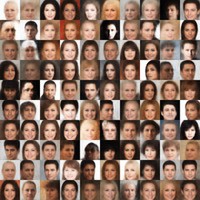}}  &
		\parbox[c]{0.18\textwidth}{\includegraphics[width=\linewidth]{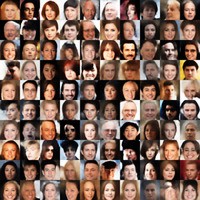}}  &
		\parbox[c]{0.18\textwidth}{\includegraphics[width=\linewidth]{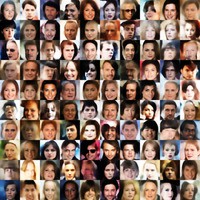}} &
		\parbox[c]{0.18\textwidth}{\includegraphics[width=\linewidth]{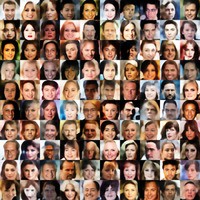}} \\
	\end{tabular}
\end{table}
\fi

\end{document}